\documentclass[preprint,12pt]{elsarticle}

\usepackage{graphicx}
\usepackage{amssymb}
\usepackage{amsthm}

\usepackage{bbm} 

\usepackage{algorithm}
\usepackage{algpseudocode}

\usepackage{lineno}

\graphicspath{
		{./figures/}
        }

\usepackage{prettyref}
\usepackage{amsmath}
\usepackage{xcolor}

\newtheorem{lemma}{Lemma}

\newrefformat{fig}{Figure \ref{#1}}
\newrefformat{tab}{Table \ref{#1}}
\newrefformat{eq}{Eq. (\ref{#1})}
\newrefformat{app}{\ref{#1}}
\newrefformat{alg}{Algorithm \ref{#1}}
\newrefformat{sec}{Section \ref{#1}}
\newrefformat{lemma}{Lemma \ref{#1}}
\newrefformat{theorem}{Theorem \ref{#1}}
\newrefformat{assumption}{Assumption \ref{#1}}

\newcommand{\m}{\mathbf{m}}
\newcommand{\w}{\mathbf{w}}
\newcommand{\x}{\mathbf{x}}
\newcommand{\y}{\mathbf{y}}
\newcommand{\B}{\mathbf{B}}
\newcommand{\Bcv}{\mathbf{\dot{B}}}

\newcommand{\D}{\mathbf{D}}
\newcommand{\Dcv}{\mathbf{\dot{D}}}

\newcommand{\Ecv}{\mathbf{\dot{E}}}
\renewcommand{\H}{\mathbf{H}}
\newcommand{\I}{\mathbf{I}}
\newcommand{\R}{\mathbb{R}}
\renewcommand{\S}{\mathbf{S}}
\newcommand{\Th}{\boldsymbol{\Theta}}
\newcommand{\Thcv}{\boldsymbol{\dot{\Theta}}}

\newcommand{\W}{\mathbf{W}}
\newcommand{\X}{\mathbf{X}}
\newcommand{\Y}{\mathbf{Y}}
\newcommand{\Yhat}{\mathbf{\hat{Y}}}
\newcommand{\Ycv}{\mathbf{\dot{Y}}}
\newcommand{\Ylda}{\mathbf{\breve{Y}}}  

\newcommand{\Te}{\text{Te}}
\newcommand{\Tr}{\text{Tr}}

\newcommand{\Beta}{\boldsymbol{\beta}}
\newcommand{\Betahat}{\boldsymbol{\hat{\beta}}}
\newcommand{\Betacv}{\boldsymbol{\dot{\beta}}}
\newcommand{\ehat}{\mathbf{\hat{e}}} 
\newcommand{\ecv}{\mathbf{\dot{e}}} 
\newcommand{\Hte}{\H_\Te} 
\newcommand{\Htrte}{\H_{\Tr,\Te}} 
\newcommand{\mbar}{\overline{\m}}
\newcommand{\mm}[1]{\m_{#1}}   
\newcommand{\Xate}{\Xa_{\text{Te}}} 
\newcommand{\yhat}{\mathbf{\hat{y}}} 
\newcommand{\ycv}{\mathbf{\dot{y}}} 
\newcommand{\yp}{\y^\sigma} 
\newcommand{\Xa}{\widetilde{\X}} 

\newcommand{\one}{\mathbbm{1}}

\newcommand{\bigO}[1]{$\mathcal{O}({#1})$}

\journal{Journal Name}

\begin{document}

\begin{frontmatter}



\title{Cross-validation in high-dimensional spaces: a lifeline for least-squares models and multi-class LDA}






\author{Matthias S. Treder}

\address{Cardiff University Brain Research Imaging Centre (CUBRIC), Cardiff University, United Kingdom}

\begin{abstract}
Least-squares models such as linear regression and  Linear Discriminant Analysis (LDA) are amongst the most popular statistical learning techniques. However, since their computation time increases cubically with the number of features, they are inefficient in high-dimensional neuroimaging datasets.
Fortunately, for k-fold cross-validation, an analytical approach has been developed that yields the exact cross-validated predictions in least-squares models without explicitly training the model. Its computation time grows with the number of test samples.
Here, this approach is systematically investigated in the context of cross-validation and permutation testing. LDA is used exemplarily but results hold for all other least-squares methods. Furthermore, a non-trivial extension to multi-class LDA is formally derived. The analytical approach is evaluated using complexity calculations, simulations, and permutation testing of an EEG/MEG dataset.
Depending on the ratio between features and samples, the analytical approach is up to 10,000x faster than the standard approach (retraining the model on each training set). This allows for a fast cross-validation of least-squares models and multi-class LDA in high-dimensional data, with obvious applications in multi-dimensional datasets, Representational Similarity Analysis, and permutation testing.
\end{abstract}


\end{frontmatter}


\section{Introduction}
Multivariate pattern analysis (MVPA) is a statistical technique in which a target variable such as a brain state or reaction time is predicted based on multivariate patterns of brain activity \cite{Mur2009}. 
The spadework in MVPA is performed by regression models if the dependent variables is continuous (e.g. reaction time), or by classifiers if the dependent variable is categorical (e.g. stimulus type) \cite{Hastie2009}. 
Due to their simplicity, relatively low computational demands, and high interpretability, least-squares models have been popular for both regression problems (linear regression, ridge regression) and for classification problems (Linear Discriminant Analysis  \cite{Fisher1936}).

The increase in storage capabilities, working memory, and computational power, and the increasing availability of high-performance compute clusters paved the way for large-scale analyses of neuroimaging data. Analyses can deal with larger throughput than ever before, such as higher field strengths in fMRI and larger number of electrodes in EEG, or simply a larger amount of derived features such as time-frequency and connectivity metrics. It is worth stressing that most neuroimaging datasets have a $P \gg N$ shape, that is the number of features P is much larger than the number of samples N. An extreme example of this is gene expression data comprising tens of thousands of genes (features) but not more than a few hundred patients (samples) \cite{Clarke2008TheData.,Wang2008ApproachesMicroarrays}. In cognitive neuroscience, the number of samples for an analysis is naturally capped by limits of experiment time and group size. For level 1 analyses, the number of trials is limited by the amount of time the subject can spend in the scanner. Even in fast-paced EEG/MEG experiments, it is very rare that more than 10,000 trials are collected. For level 2 analyses, sample size is equal to the number of subjects, which is typically less than a few hundred. 
Summarising, the principal challenge in large neuroimaging datasets is to efficiently cope with high-dimensional data.

Unfortunately, this is exactly the Achilles heel of least-squares methods (LSM) such as linear regression, ridge regression, and Linear Discriminant Analysis (LDA). The  computationally most expensive part in LSM is the inversion of the features $\times$ features scatter matrix. Computation time increases cubically with the number of features. It can therefore be intractable for even a few thousand features if a large number of training-testing iterations is needed, such as in permutation testing or in Representational Similarity Analysis \cite{Kriegeskorte2008} with many experimental conditions. This is one of the reasons that some researchers explore kernel methods such as Support Vector Machines \cite{Cortes1995Support-VectorNetworks} whose complexity grows with the number of samples rather than number of features.

Does this mean that, for all practical purposes, high-dimensional datasets are beyond reach for LSM? Fortunately, for cross-validation \cite{Lemm2011}, an alternative has been developed that addresses this issue. The analytical approach for LSM has the following  \cite{Lemm2011}useful property: instead of requiring the inversion of a feature$\times$ features scatter matrix on each training set, it instead relies on the inversion of a matrix that grows with the number of test samples. It is therefore only mildly affected by the number of features. 
For leave-one-out cross-validation, the analytical approach is well-known in the linear regression literature
\cite{Cawley2003EfficientClassifiers,Cook1982ResidualsRegression,James2013}, and it has been generalised to k-fold cross-validation \cite{Pahikkala2006FastLeast-Squares,Rao2008OnEvaluation}. 

The aim of this study is to show that LSM can successfully meet the challenges of high-dimensional data when using the analytical approach. Because of the formal equivalence between linear regression and LDA (resp. ridge regression and regularised LDA), it suffices to focus on LDA alone. All results automatically generalise to linear regression and ridge regression. Equipped with regularisation techniques such as ridge regularisation \cite{Friedman1989} or shrinkage regularisation \cite{Blankertz2011}, LDA is robust to overfitting in high-dimensional data. It often performs similarly to more sophisticated linear classifiers such as linear support vector machines (SVM) while being significantly faster to train \cite{Li2006UsingInvestigation}. 

The novel contributions in this manuscript are a detailed empirical evaluation of the analytical approach for cross-validation using simulations and complexity calculations. Furthermore, to the best of my knowledge, this is the first application in permutation testing which is a popular approach in statistical testing of classifier performance \cite{Allefeld2016ValidInference,Stelzer2013,Ojala2010,Salzberg1997,Jamalabadi2016}. It is also the first time that the approach is formally extended to multi-class LDA using an optimal scoring approach \cite{Hastie1995PenalizedAnalysis}.

The manuscript is structured as follows. Firstly, cross-validation, binary LDA and a regression formulation of LDA are introduced. Then  the analytical approach is developed for cross-validation and permutation testing, and both ridge regularisation and shrinkage regularisation are considered. Finally, it is formally extended to multi-class LDA \cite{Rao1948TheClassification,Hastie1995PenalizedAnalysis}. The computation time of the approach is then compared to the computation time of the standard approach (retraining the model for every fold) using complexity calculations, simulations, and a permutation analysis of a publicly available EEG/MEG dataset \cite{Wakeman2015ADataset}.

\section{Method}

Matrices will denoted by bold upper case letters, for instance $\X$. Vectors are denoted as bold lower case letters, $\x$, and are assumed to be column vectors. For scalars, normal font type is used. Upper case scalars are used for specifying the dimensionality of the matrices or vectors. $N$ is the number of samples, $P$ the number of features or predictors, $K$ is the number of cross-validation folds, $C$ is the number of classes, and $T$ is the total number of training-testing iterations in during permutation testing.

\subsection{Cross-validation}

Cross-validation allows to estimate predictive performance while at the same time controlling for overfitting and making efficient use of the available samples \cite{Lemm2011,Jamalabadi2016,Hastie2009}.
In k-fold cross-validation, the dataset is randomly partitioned into K equally sized folds. The classifier is trained on all but one of the folds, and then tested on the held out fold. This procedure is repeated until every fold served as test set once. Classification performance is then averaged across the test folds. To reduce the variance stemming from the random partitioning of data into folds, the cross-validation can be repeated several times, finally averaging across the repeats. 

\subsection{Linear Discriminant Analysis (LDA)}

For two classes, LDA is equivalent to Fisher Discriminant Analysis (FDA) \cite{Fisher1936,Blankertz2011,Bishop2007,Duda1998}. The multi-class case is considered further below. Geometrically speaking, LDA seeks a projection $\w$ from feature space to a 1-dimensional subspace such that the projected class means are maximally separated while at the same time the projected variance within classes is minimised \cite{Fisher1936,Blankertz2011,Bishop2007,Duda1998}. Using the LDA derivation in Duda \& Hart \cite{Duda1998} this can be formalised as:

\begin{align*}
\begin{split}
J(\w) = \frac{\w^\top\S_b\,\w}{\w^\top\S_w\,\w}
\end{split}
\end{align*}

where $\S_b\in\R^{P\times P}$ is the between-classes scatter matrix measuring the distance between the classes and $\S_w\in\R^{P\times P}$ is the within-class scatter matrix measuring the spread within each class. These quantities are defined as

\begin{equation*}
\begin{alignedat}{2}
\S_b =\ & \sum_{l\,\in\{1,2\}}N_l\,(\mm{l} -\mbar) (\mm{l} - \mbar)^\top\ \quad &&\text{(between-classes scatter)}\\
\S_w =\ & \sum_{l\,\in\{1,2\}}\sum_{i\in\mathcal{C}_l} (\x_i - \mm{l})(\x_i - \mm{l})^\top\  \quad &&\text{(within-class scatter)}
\end{alignedat}
\end{equation*}

where $\mathcal{C}_l$ is an index set representing the samples in class $l$, $N_l$ is the number of samples in class $l$, and the means are given by

\begin{equation}
\label{eq:means}
\begin{alignedat}{2}
\mm{l} =&\ \frac{1}{N_l}\sum_{i\in\mathcal{C}_l} \x_i \quad &&\text{(class mean)}\\
\mbar =&\ \frac{1}{N}\sum_{i\in\{1,2,...,N\}} \x_i \quad &&\text{(sample mean)}
\end{alignedat}
\end{equation}

Obviously, the sample mean and the two class means are related by $N\,\mbar = N_1\mm{1} + N_2\mm{2}$.
In the two-classes case, $\S_b$ further simplifies to

\begin{equation}
\label{eq:Sb-simple}
\S_b =\ \frac{N_1 N_2}{N}\ (\mm{1}-\mm{2})\,(\mm{1}-\mm{2})^\top \quad \text{(between-classes scatter)}
\end{equation}

Setting $\lambda = J(\w)$ one arrives at the generalised eigenvalue problem $\S_b\,\w = \lambda\,\S_w\,\w$. For a binary classification problem and a positive definite within-class scatter matrix, the eigenvector corresponding to the largest eigenvalue is proportional to

\begin{equation}
\label{eq:w}
\w = \S_w^{-1}\ (\mm{1} - \mm{2})\quad \text{(weight vector)}
\end{equation}

This is proved in \prettyref{lem:evproblem} in \prettyref{app:proofs}. It is expedient to define the bias as the center between the projected class means because this prevents the classifier from being biased towards one of the classes if the number of training samples per class is not equal:

\begin{equation}
\label{eq:b}
b_\text{LDA} = -\w^\top (\mm{1} - \mm{2})/2\quad \text{(bias term)}
\end{equation}

The classifier output $\hat{y}$ for a new sample $\x$ is calculated as $\hat{y} := \w^\top\x + b$. This quantity is the signed distance to the hyperplane, more generally known as \textit{decision value}. It is these decision values that are subject to the analytical approach developed below. The class labels can be derived from the sign of the decision value, with class "$+1$" for $\hat{y}\ge 0$ and class "$-1$" for $\hat{y}<0$ \cite{Bishop2007}. 

\subsection{Binary LDA as a least-squares problem}

\begin{figure}
\centering\includegraphics[width=1\linewidth]{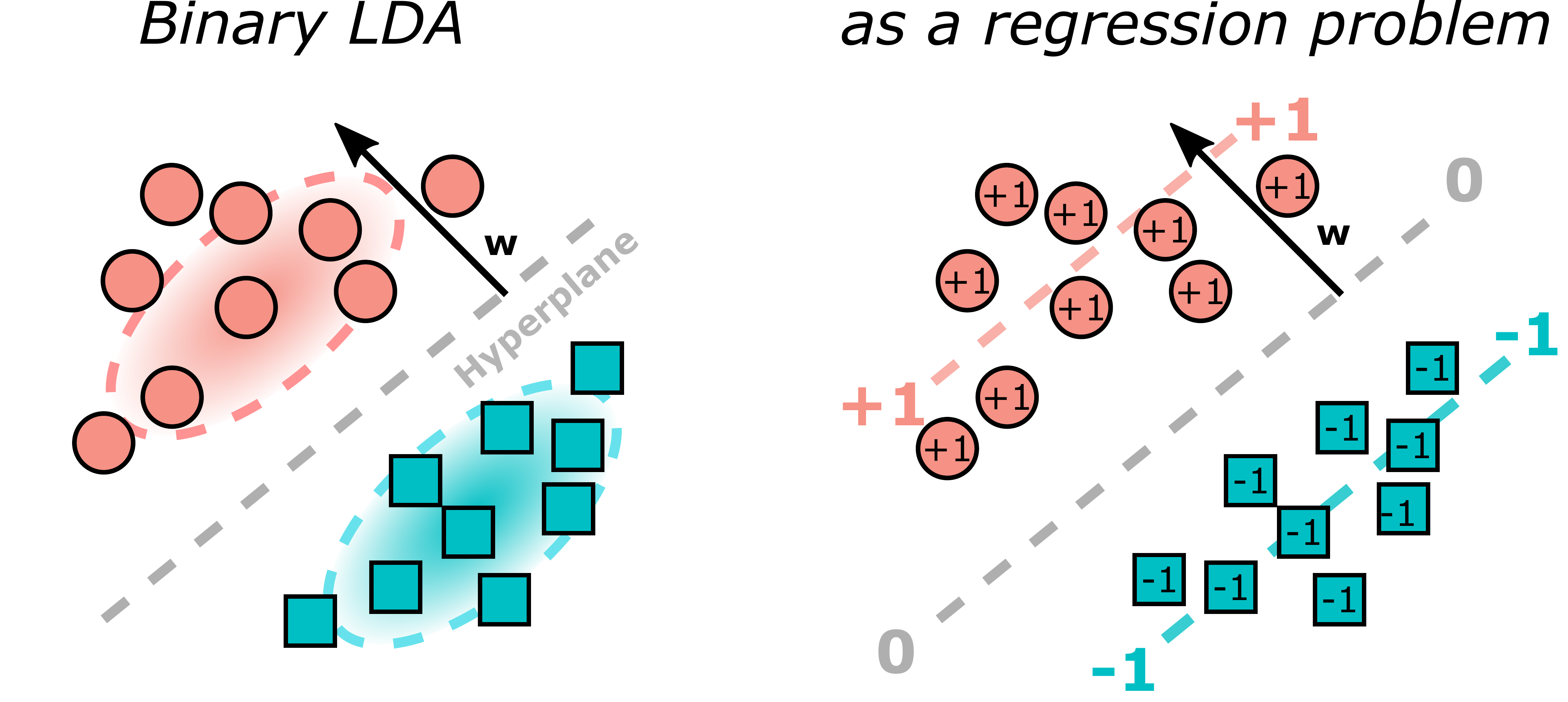}
\caption{Two equivalent perspectives on binary LDA. \textit{Left}: Classical view of LDA as a classification problem. Class distributions are modelled as multivariate Gaussian densities (indicated by the ellipses and shaded areas) and $\w$ is the normal to the optimal separating hyperplane. \textit{Right}: LDA can be framed as a regression problem by coding each class by a number (e.g. +1 and -1) and then performing linear regression using the features as predictors and class labels as response variable. Both approaches yield the same $\w$ (up to scaling).}
\label{fig:LDA_as_regression}
\end{figure}

There are several equivalent formulations of LDA.
For two classes, LDA is formally equivalent to LCMV beamforming \cite{Treder2016,vanVliet2016Single-TrialBeamformer,vanVliet2017ExploringPotentials}. Furthermore, as illustrated in \prettyref{fig:LDA_as_regression}, binary LDA can be cast as a least-squares regression problem \cite{Bishop2007,Duda1998,Mika2002KernelDiscriminants,Zhang2010RegularizedBeyond}. Let $\X\in\R^{N\times P}$ be the data matrix containing samples as rows and features as columns. Then  $\Xa\in\R^{N\times (P+1)}$ is the augmented data matrix  obtained by adding a column of 1's (for the bias term) and $\boldsymbol{\beta}$ is the regression weights vector absorbing both $\w$ and b 

\begin{align*}
\begin{split}
\Xa = [\X,\ \mathbbm{1}_N]\in\R^{N\times (P+1)}, \quad \boldsymbol{\beta} = \begin{pmatrix}\w\\ b_\text{LR}\end{pmatrix}\in\R^{P+1}
\end{split}
\end{align*}

where $\one_N$ is a vector of N ones. The bias term is denoted as $b_\text{LR}$ since it is generally different from the LDA bias term $b_\text{LDA}$. The class labels are collected in a response vector $\y\in\R^N$ that uses numerical codes (e.g. $+1$ and $-1$) for the class labels. The standard regression problem using the full dataset can then be formulated as

\begin{align}
\label{eq:LDA-regression-problem}
\hat{\Beta} = \underset{\Beta}{\text{arg min}}\ ||\Xa\,\Beta - \y||_2^2
\end{align}

with the solution given by $\hat{\Beta} = (\Xa^\top\Xa)^{-1}\ \Xa^\top\, \y$. 
Following the derivation in \prettyref{app:app-regression-binary-fda}, and assuming that the class labels are coded as $+1$ and $-1$, it can be shown that $\hat{\Beta}$ consists of the two components

\begin{align}
\begin{split}
\w\ \propto &\ \S_w^{-1}\ (\m_1 - \m_2)\\
b_\text{LR} =&\ \frac{N_1 - N_2}{N} - \mbar^\top\w
\end{split}
\end{align}

In other words, the solution for $\w$ using linear regression is proportional to the LDA solution given in \prettyref{eq:w}. Since scaling does not affect classification performance, one can say that the solutions are identical. Unless the classes have equal proportions of samples ($N_1=N_2$), the bias term $b_\text{LR}$ differs from the common choice presented in \prettyref{eq:b}. However, class proportions and the exact numerical coding of the classes in $\y$ do not affect the direction of $\w$. This is shown in \prettyref{app:app-regression-binary-fda}.





\subsection{An analytical approach to cross-validation for least-squares methods}

\begin{figure}
\centering\includegraphics[width=1\linewidth]{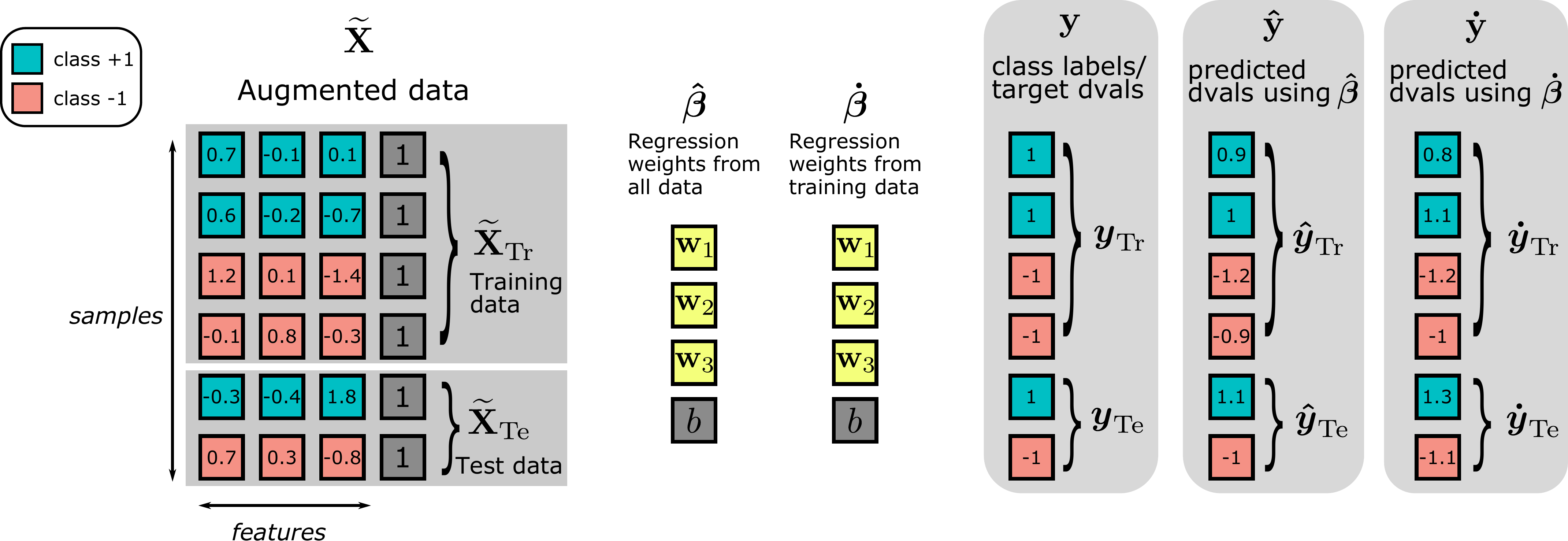}
\caption{Visual depiction of important variables used in the derivation. Decision values are abbreviated as 'dvals'.}
\label{fig:symbol_explanation}
\end{figure}

In this section, the analytical approach to cross-validation is introduced for binary LDA. For linear regression and ridge regression, the approach is identical, with $\y$ simply being continuous a response variable instead of a vector of class labels. The approach has been introduced before for leave-one-out cross-validation \cite{Cawley2003EfficientClassifiers,Cook1982ResidualsRegression,James2013} and k-fold cross-validation \cite{Pahikkala2006FastLeast-Squares,Rao2008OnEvaluation} but without the detailed derivation provided here. 

Let Tr $\subset\{1,2,\dots,N\}$ be the indices of the training samples, and Te $\subset\{1,2,\dots,N\}$ be the indices of the test samples.
Let $\y\in\{-1,+1\}^N$ be the vector of class labels. In the regression framework, these class labels serve as target decision values. The decision values obtained from the classifier trained on the full dataset are denoted as $\yhat$. The cross-validated decision values obtained from a classifier trained on only the training set and then tested on the independent test set are denoted as $\ycv$. $\X_\Tr$ resp. $\X_\Te$, and $\y_\Tr$ resp. $\y_\Te$ refer to the submatrix or subvector corresponding the training resp. test samples. To ease reading of the formulas, some important quantities used in the derivations are depicted in \prettyref{fig:symbol_explanation}.

\subsubsection{Basic idea}

In the regression framework, training the classifier is equivalent to calculating the vector of regression weights $\Betacv$ from the training data

\begin{align}
\label{eq:Beta-train}
\Betacv = (\Xa_\Tr^\top\,\Xa_\Tr)^{-1}\ \Xa_\Tr^\top\, \y_\Tr \quad\text{(model based on training data)}
\end{align}

In k-fold cross-validation, this process is repeated for each of the K training folds. However, as will be shown next, it suffices to train only \textit{one} model using all available data

\begin{align*}
\Betahat = (\Xa^\top\,\Xa)^{-1}\ \Xa^\top\, \y\quad\text{(model based on all data)}
\end{align*}

and then obtain the cross-validated decision values directly via an analytical approach.

\subsubsection{Hat matrix}

The hat matrix $\H\in\R^{N\times N}$ is defined as

\begin{align}
\label{eq:hat-matrix}
\H = \Xa\ (\Xa^\top\Xa)^{-1}\ \Xa^\top \quad \text{(hat matrix)}
\end{align}

It is a quantity well-known in linear regression \cite{Hoaglin1978TheANOVA}. Its name stems from the fact that it "puts the hat" onto the response vector $\y$ by mapping the true responses onto the predicted responses $\hat{\y} = \H\y$. The submatrix $\Hte$ is obtained from $\H$ by selecting only the rows and columns that correspond to test samples. $\Htrte$ is obtained from $\H$ by selecting the rows corresponding to training samples and the columns corresponding to test samples.
As will be seen below, the hat matrix arises naturally during updating.

\subsubsection{Updating $\Xa_\Tr^\top\,\y_\Tr$}

A formula that will prove useful later on is the product $\Xa_\Tr^\top\,\y_\Tr$ which can be obtained as follows

\begin{align}
\label{eq:update-Xy}
\Xa_\Tr^\top\,\y_\Tr = \Xa^\top\y - \Xate^\top\,\y_\Te.
\end{align}

\subsubsection{Updating the inverse scatter matrix}

Similarly, the scatter matrix on the training data can be obtained from the full scatter matrix by removing the scatter corresponding to the test samples, 

\begin{align}
\label{eq:update-scatter}
\Xa_\Tr^\top\,\Xa_\Tr = \Xa^\top\,\Xa - \Xate^\top\,\Xate.
\end{align}

Suppose that $\Xa^\top\,\Xa$ and its inverse, denoted as $\S := (\Xa^\top\,\Xa)^{-1}$ have already been calculated and let $\I$ be the identity matrix. The matrix inversion lemma (a.k.a. Sherman-Morrison-Woodbury formula) can be used to update the inverse scatter matrix on the training data as

\begin{align}
\begin{split}
\label{eq:update-inverse}
(\Xa_\Tr^\top\,\Xa_\Tr)^{-1} =&\ (\Xa^\top\,\Xa - \Xate^\top\,\Xate)^{-1}\\
=&\ \S + \S\ \Xate^\top\,(\I - \Xate\,\S\,\Xate^\top)^{-1}\Xate\,\S\\
=&\ \S + \S\ \Xate^\top\,(\I - \Hte)^{-1}\Xate\,\S
\end{split}
\end{align}

This solution circumvents the explicit inversion of the scatter matrix, but it still involves a number of matrix multiplications. It therefore only serves as an intermediate result.

\subsubsection{Updating the weights}

To calculate the weights on the training data, \prettyref{eq:update-Xy} and \prettyref{eq:update-inverse}  can be plugged into \prettyref{eq:Beta-train}. This yields

\begin{align}
\begin{split}
\label{eq:update-beta}
\Betacv =&\ \left(\S + \S\ \Xate^\top\,(\I - \Hte)^{-1}\Xate\,\S\right)\ (\Xa^\top\y - \Xa_\Te^\top\,\y_\Te)\\
=&\ \Betahat - \left(\S\ \Xate^\top\,(\I - \Hte)^{-1}\right)\ \left(\,[\I-\Hte]\,\y_\Te - \Xate\,\Betahat + \Hte\,\y_\Te\right)\\
=&\ \Betahat - \left(\S\ \Xate^\top\,(\I - \Hte)^{-1}\right)\ \left(\,\y_\Te - \Xate\,\Betahat\right)\\
=&\ \Betahat - \left(\S\ \Xate^\top\,(\I - \Hte)^{-1}\right)\ \left(\,\y_\Te - \yhat_\Te \right)\\
\end{split}
\end{align}

where $\ehat_\Te  := \y_\Te - \yhat_\Te$ is the estimation error on the test samples between the correct and the predicted decision values using a model trained on the full dataset.
As will be seen next, $\Betacv$ does not need to be calculated explicitly. 

\subsubsection{Updating the decision values}

The goal is to derive the cross-validated decision values for the test samples denoted as $\ycv_\Te$. As an intermediate step, the corresponding estimation error is calculated first.

\begin{align}
\begin{split}
\label{eq:ecvte}
\ecv_\Te= \y_\Te - \ycv_\Te\quad\text{(cross-validated estimation error)}
\end{split}
\end{align}

Inserting \prettyref{eq:update-beta} then leads to the desired analytical approach

\begin{align}
\begin{split}
\label{eq:update-rule}
\ecv_\Te =&\ \y_\Te - \Xa_\Te\,\Betacv\\
=&\ \underbrace{\y_\Te - \yhat_\Te}_{=\ehat_\Te} + \Hte (\I - \Hte)^{-1}\,\ehat_\Te\\
=&\ (\I - \Hte + \Hte)\,(\I - \Hte)^{-1}\,\ehat_\Te\\
=&\ (\I - \Hte)^{-1}\,\ehat_\Te \quad\quad\text{(analytical approach)}
\end{split}
\end{align}

It is now easy to obtain the cross-validated decision values on the test set by simply solving \prettyref{eq:ecvte} for $\ycv_\Te$. Finally, these decision values can be used to calculate classification accuracy, AUC, or any other desired metric of classification performance.

\subsection{Adjusting the bias term}

The bias term resulting from the regression approach does not generally coincide with the bias term used in LDA. However, the bias terms to coincide if $N_1 = N_2$. Hence, for unbalanced data, undersampling of the majority class or oversampling of the minority class is a remedy. Alternatively, if area under the ROC curve (AUC) is used as classifier performance metric, the bias term is irrelevant. 

If it is not possible to use one of these approaches, the bias term needs to be adjusted. To this end, the class means and the sample mean on the training need to be calculated and projected onto $\w$ (see definition of $b_\text{LDA}$ and $b_\text{LR}$). Fortunately, it is not required to explicitly calculate $\w$. Instead, one can determine the decision values of the cross-validated model on the training set, $\ycv_\Tr$, and then calculate $b_\text{LR}$ and $b_\text{LDA}$ directly. This is achieved by applying \prettyref{eq:update-rule} to the training data:

\begin{align}
\begin{split}
\label{eq:fixb}
\ecv_\Tr =&\ \ehat_\Tr + \H_{\Tr,\Te}\, (\I - \Hte)^{-1}\, \ehat_\Te\\
\ycv_\Tr =&\ \y_\Tr - \ecv_\Tr
\end{split}
\end{align}

Finally, the operation $\ycv_\Te \leftarrow \ycv_\Te - b_\text{LR} + b_\text{LDA}$ adjusts the bias.

\subsubsection{Summary: analytical approach}

It has been shown that in k-fold cross-validation, it suffices to train just one regression model on the whole dataset. By evaluating \prettyref{eq:update-rule} the decision values for each of the folds are obtained directly without explicitly calculating any of the K models.
\subsection{Regularisation}

Since neuroimaging data is often low-dimensional, or the number of samples is smaller than the number of features, the within-class scatter matrix tends to be ill-conditioned. Two similar regularisation approaches have been explored in the literature. They are equivalent in that they define the same family of classifiers (up to scaling of $\w$).

\subsubsection{Ridge regularisation}

A multiple of the identity matrix is added to the within-class scatter matrix \cite{Friedman1989,Tikhonov1977SolutionsProblems,Ng2004FeatureInvariance}. The regularised within-class scatter matrix is  $\S_w + \lambda \I$, 
%
%
where $\lambda\in[0,\infty]$ is the regularisation term and $\I$ is the identity matrix. $\lambda = 0$ yields the ordinary, unregularised solution. The ridge solution for $\w$ is then given by

\begin{equation}
\label{eq:w-ridge}
\w = (\S_w + \lambda \I)^{-1}\ (\mm{1} - \mm{2})
\end{equation}

In \prettyref{app:ridge} it is proven that, in the regression framework, the corresponding solution is given by

\begin{align}
\label{eq:LDA-ridge-regression-solution}
\hat{\Beta} = (\Xa^\top\Xa + \lambda \I_0)^{-1}\ \Xa^\top\, \y
\end{align}

where $\I_0\in\R^{(P+1)\times(P+1)}$ is a diagonal matrix that is identical to the identity matrix except for the last element which is 0 instead of 1. This construction assures that the bias term corresponding to the last entry is not subjected to regularisation.

Analogous to \prettyref{eq:update-scatter}, the regularised scatter matrix for the training data can be obtained as an update on the full model:

\begin{align*}
\Xa_\Tr^\top\,\Xa_\Tr + \lambda\I_0 = \underbrace{\Xa^\top\,\Xa + \lambda\I_0}_{\text{full scatter}} - \underbrace{\Xate^\top\,\Xate}_{\text{update}}
\end{align*}

After redefining $\S := (\Xa^\top\,\Xa + \lambda\I_0)^{-1}$ and correspondingly including the regularisation term in the hat matrix $\H = \Xa\ (\Xa^\top\Xa + \lambda\I_0)^{-1}\ \Xa^\top$, the analytical approach is identical to \prettyref{eq:update-rule}.

\subsubsection{Shrinkage regularisation}

The within-class scatter matrix is replaced by a convex combination of the empirical covariance matrix and a scaled identity matrix, $(1-\lambda)\ \S_w + \lambda\nu \I$, where $\nu = \text{trace}(\S_w)/P$ is a scaling parameter that equalises the traces of $\S_w$ and $\nu\I$, and $\lambda\in[0,1]$ \cite{Blankertz2011}. Unfortunately, shrinkage regularisation does not allow for simple low-rank updates as before. This can be seen when one inspects an update of the regularised scatter matrix

\begin{align*}
(1-\lambda)\Xa_\Tr^\top\Xa_\Tr + \lambda\nu_\Tr\I_0 = \underbrace{(1-\lambda)\Xa^\top\Xa + \lambda\nu\I_0}_{\text{full scatter}} - \underbrace{\left((1-\lambda)\Xate^\top\Xate + \lambda(\nu-\nu_\Tr)\I_0\right)}_{\text{update}}
\end{align*}

where $\nu$ is the scaling calculated on the full dataset and $\nu_\Tr$ is the scaling calculated on the training data. The problem is that it is necessary to update the regularisation term as well. This is caused by the scaling factor $\nu_\Tr$, which changes for each training set, thereby changing the amount of regularisation. This turns a low-rank update into a full rank update, precluding significant performance gains by updating.

For this reason, it is recommended to resort to ridge regularisation. If a researcher is used to work with shrinkage, the following simple relation can be used to transform a given shrinkage parameter $\lambda_\text{shrink}$ into a corresponding ridge parameter $\lambda_\text{ridge}$. Given a fixed value for $\lambda_\text{shrink}$, the goal is to find a $\lambda_\text{ridge}$ such that the regularised scatter matrices are proportional:

\begin{align*}
(1-\lambda_\text{shrink})\ \Xa^\top\,\Xa + \lambda_\text{shrink}\,\nu\,\I_0\quad  \overset{!}{\propto} \quad \Xa^\top\,\Xa + \lambda_\text{ridge}\,\I_0
\end{align*}

Obviously, this relation holds when the ridge parameter is defined as

\begin{align}
\label{eq:shrink-to-ridge-solution}
\lambda_\text{ridge} = \frac{\lambda_\text{shrink}}{1-\lambda_\text{shrink}}\ \nu
\end{align}

\subsection{Using the analytical approach for permutation testing}

The hat matrix $\H$ is invariant under class label permutations because it depends on the features alone. Consequently, it does not need to be recalculated when the class labels are permuted. Let the permuted class labels be denoted as $\yp$. If $\y$ and $\yhat$ are adjusted accordingly

\begin{align*}
\begin{split}
\y \leftarrow&\ \yp\\
\yhat \leftarrow&\  \H\,\yp
\end{split}
\end{align*}

the formulas in the previous section directly apply. They are compiled in \prettyref{alg:alg1}.

\begin{algorithm}
\caption{Fast cross-validation and permutations for binary LDA}
\label{alg:alg1}
\begin{algorithmic}
\State $\H \gets \Xa\ (\Xa^\top\Xa + \lambda\I_0)^{-1}\ \Xa^\top$
\ForAll{permutations $\sigma$}
	\State $\y \gets\yp$
	\State $\yhat \gets\H\,\yp$
    \ForAll{test sets Te}
    	\State $\ecv_\Te \gets  (\I - \Hte)^{-1}\,\ehat_\Te$
        \State $\ycv_\Te \gets  \y_\Te - \ecv_\Te$
		\State Calculate classification performance on current test set
	\EndFor
    \State Average classification performances across test sets
\EndFor
\State Output: classification performance for each permutation
\end{algorithmic}
\end{algorithm}

\subsection{Multi-class LDA}

Multi-class LDA is the generalisation of binary LDA to more than two classes. Like binary LDA, it involves a projection step and a thresholding step. In the projection step, the data is mapped onto a $(C-1)$-dimensional subspace, where $C$ is the number of classes. In the second step, a new sample is assigned to the class with the closest class centroid. LDA thus acts as a prototype classifier.
The scatter matrices are calculated as before, but now information is pooled across all classes.

\begin{equation*}
\begin{alignedat}{2}
\S_b =\ & \sum_{j\,\in\{1,2,...,C\}}n_j\,(\mm{j} -\mbar) (\mm{j} - \mbar)^\top\ \quad &&\text{(between-classes scatter)}\\
\S_w =\ & \sum_{j\,\in\{1,2,...,C\}}\sum_{i\in\mathcal{C}_j} (\x_i - \mm{j})(\x_i - \mm{j})^\top\  \quad &&\text{(within-class scatter)}\\
\end{alignedat}
\end{equation*}

Assuming that there are more features than classes, the between-classes scatter matrix $\S_b$ has rank $C-1$. Consequently, there are multiple non-trivial solutions that again can be obtained via the generalised eigenvalue problem

\begin{align}
\label{eq:LDA-eigenvalue-multiclass}
\S_b\,\W = \S_w\,\W\mathbf{\Lambda}
\end{align}

where $\mathbf{\Lambda}$ is a diagonal matrix of eigenvalues. A set of discriminant coordinates is obtained corresponding to non-zero eigenvalues of the eigenvalue problem. They are collected in a matrix $\W\in\R^{P\times(C-1)}$ and scaled such that $\W^\top\S_w\W = \I$ \cite{Bishop2007}. As for binary LDA, ridge regularisation can be applied to the within-class scatter matrix by replacing $\S_w$ by $\S_w+\lambda\I$ \cite{Friedman1989}.

\subsection{Multi-class LDA in a regression framework}

Unfortunately, multi-class LDA is not equivalent to multivariate linear regression using a class indicator matrix as response matrix \cite{Hastie2009}. Nevertheless, there is a close relationship between both approaches  \cite{Hastie1995PenalizedAnalysis,Hastie2009,Ye2007LeastAnalysis,Park2005ASolution}.
A useful characterisation is given in \citet{Hastie1995PenalizedAnalysis}. They show that multi-class LDA is equivalent to optimal scoring (OS) wherein regression is performed using a response vector with optimal numerical scores for each class. Finding the optimal scores is an optimisation problem that is solved jointly with the regression problem.
Let $\Y\in\R^{N\times C}$ be the class indicator matrix whose (i,j)-th element is defined as

\begin{equation*}
   \Y_{ij} =
   \begin{cases}
     1 & \text{if sample $i$ belongs to class $j$}\\
     0 &\text{otherwise}
   \end{cases}
\end{equation*}

Let  $\boldsymbol{\theta}\in\R^C$ be the vector containing the optimal scores. Then the response vector of optimal scores can be written as $\Y\boldsymbol{\theta}$, and the optimal scoring problem is given by

\begin{align*}
\underset{\Beta,\mathbf{\theta}}{\text{arg min}}\ ||\Xa\,\Beta - \Y\boldsymbol{\theta}||_2^2 \quad\text{(optimal scoring)}
\end{align*}

where $\Beta$ and $\boldsymbol{\theta}$ are jointly optimised. The additional constraint $N^{-1}||\Y\boldsymbol{\theta}||^2 = 1$ avoids trivial solutions.
\citet{Hastie1995PenalizedAnalysis} show that this optimisation problem can be broken up into two successive steps.\\

\textit{Step 1}: A multivariate regression is performed on the class indicator matrix $\widetilde{\B} = \text{arg min}\ ||\Xa\,\widetilde{\B} - \Y||_F^2$, 
where $||\cdot||_F$ is the Frobenius norm. The result is a matrix of regression weights $\widetilde{\B}\in\R^{(P+1)\times C}$, where each column of regression weights corresponds to the respective column of $\Y$. This yields the matrix of regression fits $\Yhat = \H\Y$. \\

\textit{Step 2}: The optimal score vector is found via an eigendecomposition of $\Yhat^\top\Y$. Let $\Th\in\R^{C\times (C-1)}$ be the eigenvectors of this decomposition, also called optimal scores, where the column corresponding to the trivial eigenvalue 0 (if $\Xa$ is centered) or 1 (if $\Xa$ is not centered) has been removed. Let $\alpha_1^2,\alpha_2^2,...,\alpha_{C-1}^2$ be the corresponding eigenvalues. Let $\B$ be the submatrix of $\widetilde{\B}$ with the last row (bias term) omitted. Then the columns of $\B\Th$ point in the same directions as the discriminant coordinates obtained in multi-class LDA but their scaling differs. To scale the discriminant coordinates, they are right-multiplied with the diagonal matrix

\begin{equation*}
   \mathbf{D} = \sqrt{N}^{-1} \begin{pmatrix} %
   \sqrt{\alpha_1^2(1-\alpha_1^2)}   	& 	0 			& \dots &	0\\
   0		   	& 	\sqrt{\alpha_2^2(1-\alpha_2^2)}  	& \dots &	0\\
   \vdots   	& 	\vdots 	& \ddots &	0\\
   0		  	& 	0 	& \dots &	\sqrt{\alpha_{C-1}^2(1-\alpha_{C-1}^2)}
   \end{pmatrix}^{-1}
\end{equation*}

Note that the normalisation $\sqrt{N}^{-1}$ does not appear in the original definition of the scaling matrix (\cite{Hastie1995PenalizedAnalysis}, p. 83) but is necessary here because the multi-class LDA has been calculated using the within-class scatter matrix. In contrast, Hastie et al. use the covariance matrix which differs by a scaling factor of $N$. As main result of their derivation, the relationship between the  discriminant coordinates $\W$ in \prettyref{eq:LDA-eigenvalue-multiclass} and the optimal scoring results is given by

\begin{equation}
\W = \B\Th\D
\label{eq:LDA-as-OS}
\end{equation}

\subsection{The analytical approach for multi-class LDA}

How can these findings be used to develop an analytical approach for multi-class LDA?
Starting from \prettyref{eq:LDA-as-OS}, a dot is used to indicate that the matrices have been estimated using the training data. Left-multiplication with the test data then yields

\begin{align*}
\begin{split}
\X_\Te\,\mathbf{\dot{W}} =&\ \X_\Te\,\Bcv\Thcv\Dcv\\
\Leftrightarrow\ \Ylda_\Te =&\ \Ycv_\Te\,\Thcv\Dcv
\end{split}
\end{align*}

where $\Ylda_\Te$ is used to denote the desired discriminant scores for the test data (obtained in step 2). This notation is necessary to differentiate them from the cross-validated regression fits $\Ycv$ (obtained in step 1).  After calculating $\Yhat = \H\Y$, $\Ycv_\Tr$ and $\Ycv_\Te$ can be obtained by applying \prettyref{eq:update-rule} and \prettyref{eq:fixb} using the matrix of estimation errors $\mathbf{\hat{E}} = \Y - \Yhat$. $\Thcv$ and $\Dcv$ are  then obtained via the eigenanalysis \texttt{eig}$(\Ycv_\Tr^\top\, \Y_\Tr / N_\Tr)$ on the training data.
Note that in practice, the augmented data matrix $\Xa$ and the regression weights $\widetilde{\B}$ can be used. The classification results are equivalent since the distance of a sample to the class centroids is unaffected by the constant shift incurred by the bias term.

Concluding, an analytical approach for step 1 of OS has been developed. There is no straightforward way to update the eigenvalue decomposition in step 2. However, the eigenanalysed matrix is of dimensions $C\times C$, so for most practical applications the computational costs are negligible. \prettyref{alg:alg-multiclass} compiles these results.

\begin{algorithm}
\caption{Fast cross-validation and permutations for multi-class LDA}
\label{alg:alg-multiclass}
\begin{algorithmic}
\State $\H \gets \Xa\ (\Xa^\top\Xa + \lambda\I_0)^{-1}\ \Xa^\top$
\ForAll{permutations $\sigma$}
	\State $\Y \gets\Y^\sigma$
	\State $\Yhat \gets\H\,\Y^\sigma$
    \ForAll{test sets Te}
    	\State (step 1)
        \State $\Ycv_\Tr \gets  \Y_\Tr - \Ecv_\Tr$
        \State $\Ycv_\Te \gets  \Y_\Te - \Ecv_\Te$
    	\State (step 2)
        \State $(\Thcv,(\alpha_1^2,\alpha_2^2,...)) \gets  \text{\texttt{eig}}(\Ycv_\Tr^\top\, \Y_\Tr / N_\Tr)$
        \State $\Ylda_\Te \gets \Ycv_\Te\,\Thcv\Dcv$
		\State Calculate classification performance on current test set
	\EndFor
    \State Average classification performances across test sets
\EndFor
\State Output: classification performance for each permutation
\end{algorithmic}
\end{algorithm}


\subsection{Computational complexity of the analytical approach}

\begin{table}
   \centering
   \begin{tabular}{lll}\hline
     Method 	&  Classes 	& Complexity \\
     \hline\\
     Standard 		& Binary 	& \bigO{KNP^2 + KP^3}\\
     analytical approach 	& Binary 	& \bigO{KN^3}\\
     Standard 		& Multi-class 	& \bigO{KNP^2 + KCP^2 + TP^3}\\
     analytical approach 	& Multi-class 	& \bigO{KN^3 C}\\
     \hline
   \end{tabular}
   \caption{Computational complexity of training LDA classifiers in a permutation testing regime. The standard approach is compared with the analytical approach presented in this paper. $K$: \#folds; $N$: \#samples; $P$: \#features; $C$: \#classes}
	\label{tab:complexity}
\end{table}

The asymptotic computational complexity for classifier validation using cross-validation is quantified in terms of floating point operations. The analytical approach developed in this paper is compared to the standard approach wherein a classifier is trained from scratch on every training set. Regularisation is not considered separately since the addition of the regularisation term inflicts negligible costs and is the same in both algorithms. For simplicity, the complexity is given in terms of the standard textbook algorithms. Speed-ups can of course be achieved using more sophisticated algorithms. The complexity calculations are summarised in \prettyref{tab:complexity}.

\subsubsection{Binary LDA}

For training a single binary LDA classifier, two class means need to be calculated which involves adding up training samples and features and dividing two times, leading to a complexity of \bigO{NP}, where $P$ is the number of features. Calculating the within-class scatter matrix requires  $N(P+P^2)$ steps, where the $P$ is for subtracting the class mean and the $P^2$ is for calculating the outer vector product (\bigO{NP^2}). Instead of then calculating the inverse of the within-class scatter matrix, one can solve the system of linear equations $\S_w\,\w = \m_1 - \m_2$ (\bigO{P^3}). Calculation of the bias $b_\text{LDA}$ requires \bigO{P}. Taken together, the complexity for training a single classifier amounts to \bigO{NP^2 + P^3}. This process is repeated $K$ times, $K$ being the number of folds.
This yields an overall complexity of \bigO{KNP^2 + KP^3} since the lower-order terms can be ignored for asymptotic complexity.

For the analytical approach based on the regression approach, the hat matrix needs to be calculated initially at a complexity of \bigO{N^2P + NP^2 + P^3}. Then \prettyref{eq:update-rule} needs to be evaluated for each training iteration  (\bigO{KN^3}). If the bias term needs to be corrected, operations at \bigO{KN^2} are required. This yields an asymptotic complexity of \bigO{KN^3} for the analytical approach.

\subsubsection{Multi-class LDA}

The asymptotic complexity for training a single multi-class LDA classifier is provided first. Calculating means for each of the $C$ classes, involves adding up training samples and features and dividing $C$ times (\bigO{NP} + \bigO{CP}).
Calculating the within-class scatter matrix is equal to the binary case (\bigO{NP^2}). Calculating the between-classes scatter matrix involves calculating $C$ outer vector products \bigO{CP^2}. The generalised eigenvalue decomposition has an overall complexity of \bigO{P^3}.  Repeating this process $K$ times yields an overall complexity of \bigO{KNP^2 + KCP^2 + KP^3}.

For the analytical approach based on the optimal scoring approach, the hat matrix needs to be calculated initially at a complexity of \bigO{N^2P + NP^2 + P^3}. Obtaining the cross-validated regression fits $\Ycv_\Tr$ and $\Ycv_\Te$ involves a complexity of \bigO{KN^3 C} each. This is followed by the calculation and eigendecomposition of $\Ycv_\Tr^\top\, \Y_\Tr$ (\bigO{KC^2 N + KN^3 C}). Finally the discriminant scores are calculated (\bigO{KC^2 N}). This yields an asymptotic complexity of \bigO{KN^3 C}.

\subsection{Simulations}

To vet the analytical approach its efficacy is compared to the standard approach using simulated data. The data is created as follows: Each class centroid is randomly placed on the surface of a unit hypersphere in feature space. A common covariance matrix is randomly sampled from a Wishart distribution. Samples are then created by randomly sampling from a multivariate normal distribution parameterised by the corresponding class centroid and the common covariance matrix.

The number of features was varied from 10 to 1000 in 40 logarithmic steps. For binary LDA, cross-validation was performed using 5 folds, 10 folds, 20 folds, and leave-one-out. Simulations were run separately for 100 and 1000 samples. Simulations were run with 10-fold cross-validation and 100, 1000, or 10,000 permutations. The number of samples and the number of features were set to either 100 or 1000. For every combination of parameters,  the simulation was repeated 20 times.

For multi-class LDA, 10-fold cross-validation was used with data being split into 5 classes or 10 classes with equal class proportions. For cross-validation, the number of samples was either 100 or 1000. For permutations, the number of features was fixed to 100 or 1000. The number of permutations was limited to 10 or 100 to keep overall  computation time tractable. For every combination of parameters, the simulation was repeated 20 times for cross-validation. For permutations, it was repeated 10 times.

Both binary and multi-class LDA update rules were compared to the vanilla approach wherein the classifier is trained on each training set and then applied to the test set. Different datasets and different folds were randomly created for each choice of parameters. However, for each of the two methods (analytical approach vs classical approach) the random seed was reset to assure equal data and equal folds.
All analyses were performed in MATLAB (Natick, USA). The \texttt{tic} and \texttt{toc} functions were used to measure the total computation time for cross-validation and permutation testing iterations. As target measure, \textit{relative effiency} was computed, defined as

\[
\text{Relative efficiency} = log_\text{10}\ \frac{\text{time(standard approach)}}{\text{time(analytical approach)}}
\]

This quantity has a simple interpretation in terms of \textit{orders of magnitude of improvement in computation time} of the analytical approach over the standard method. For instance, a value of 0 means that both methods are at parity. A value of 1 means that the analytical approach is 10 times faster than the standard method, a value of 2 means that the analytical approach is $10^2 = 100$ times faster, and so on. Simulations were run on a Thinkpad X1 Carbon with 16 GB of RAM and an Intel Core i7-6600U CPU @ 2.60GHz $\times$ 4 processor.

\subsection{EEG data}

The analytical approach was applied to a publicly available multi-modal dataset of participants watching greyscale images of faces and scrambled faces \cite{Wakeman2015ADataset}. The 16 EEG/MEG datasets with a total of 380 EEG/MEG channels were read into MATLAB using FieldTrip \cite{Oostenveld2011}. Epochs were created from -0.5 s to 1 s relative to image onset, and the pre-stimulus interval was used for baseline correction. Finally, data were downsampled to 200 Hz. The type of stimulus (face vs scrambled) was used as class label for binary LDA. For multi-class LDA, the face stimuli were further split in order to create 3 classes.
The total number of  trials varied across subjects, with 787 trials on average.

For both binary and multi-class LDA, two different analyses were conducted. In the first case, classification was performed separately for every time point across time interval  ranging from -0.5 s to 1 s. At each time point, 100 permutations were conducted with shuffled class labels and a 10-fold cross-validation in each permutation. The amplitudes in each channel and were used as features (380 features). In the second case, the post-stimulus interval was divided into successive, non-overlapping windows. The amplitudes in each channel were averaged within these windows and then all averaged amplitudes were concatenated to a single feature vector. For binary LDA, 100 ms windows were used ($10 * 380 = 3800$ features). For multi-class LDA, 200 ms windows were used ($5 * 380 = 1900$ features).

Analyses were run on the high-performance cluster at the Cardiff University Brain Research Imaging Centre (CUBRIC). Each compute node consists of 12 cores with 192 GB RAM and an Intel(R) Xeon(R) X5660 CPU running at 2.80GHz.

\subsection{Software}

MATLAB implementations of the analytical approach and the scripts reproducing the results are publicly available on GitHub (\texttt{github.com/treder/Fast-Least-Squares}). Note that parts of the code require the MVPA-Light toolbox (\texttt{github.com/treder/MVPA-Light}) to run.

\section{Results}

\subsection{Simulations}

\begin{figure}
\centering\includegraphics[width=1\linewidth]{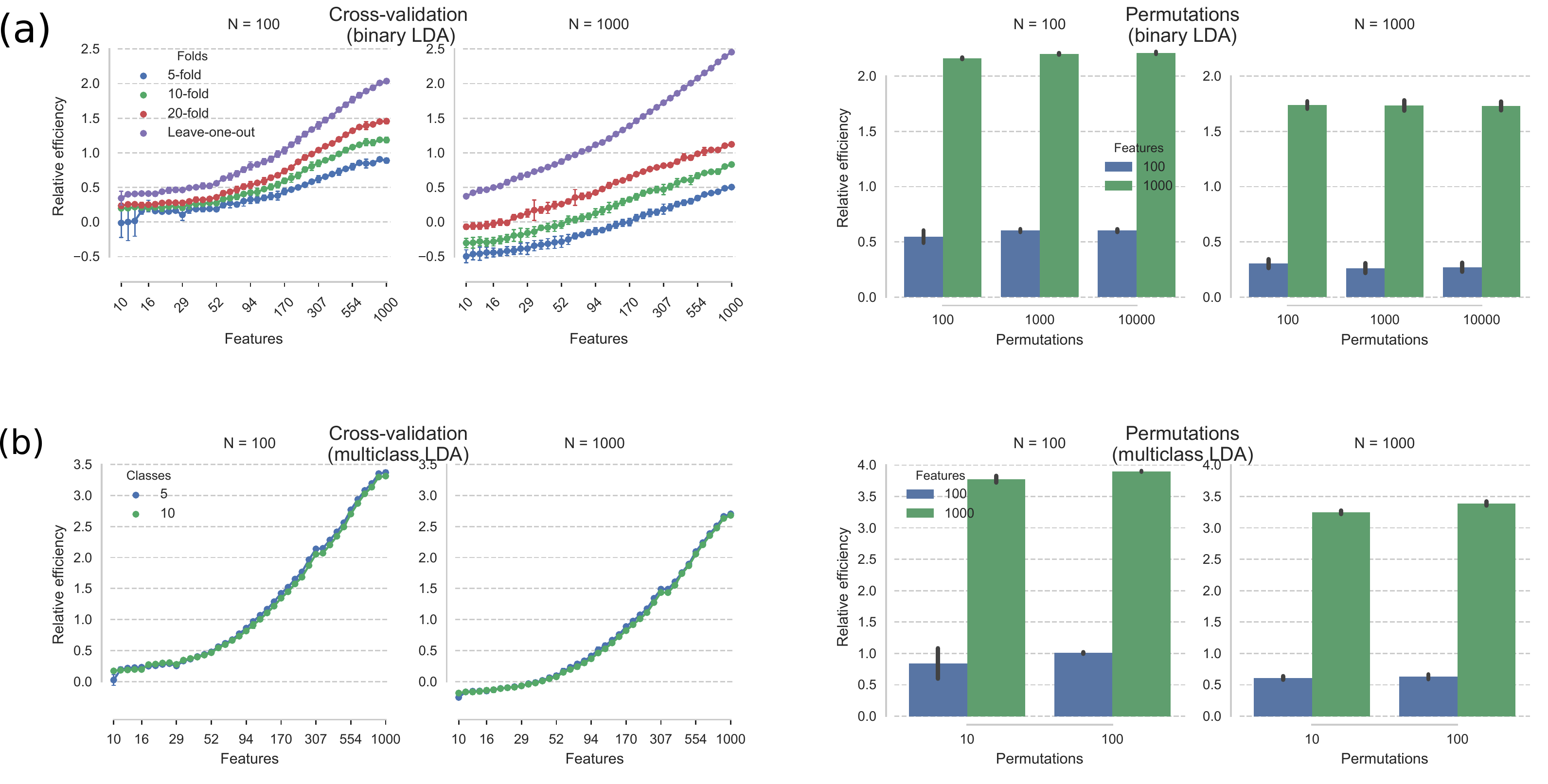}
\caption{Results of the simulations denoted in terms of relative efficiency. A relative efficiency of 1 means that the analytical approach is 10x faster than the standard approach, 2 means 100x faster, and 3 means 1000x faster. (a) Binary LDA. (b) Multi-class LDA.}
\label{fig:simulation_results}
\end{figure}

\textit{Binary LDA}. A three-way analysis of variance (ANOVA) was run on the cross-validation analysis (\prettyref{fig:simulation_results}, top left). A continuous variable (features) and two categorical variables, samples N (100 or 1000) and folds (5, 10, 20, leave-one-out), were used as predictors, and relative efficiency was used as dependent variable.  There were significant main effects of features ($F = 32051.69; p < .001$), N ($F = 1316.19; p < .001$), and folds ($F = 3119.03; p < .001$). Furthermore, the effect of features increased with folds (features $\times$ N, $F = 806.49; p < .001$), and there was an N $\times$ folds interaction ($F = 812.7; p < .001$). Furthermore, there was a three-way interaction N $\times$ folds $\times$ features ($F = 37.16; p < .001$).

A separate three-way ANOVA was performed on the permutations data (\prettyref{fig:simulation_results}, top right) using N, permutations, and features as predictors. There were significant main effects for N ($F = 6899.92; p < .001$), permutations ($F = 4.35; p = .014$), and features ($F = 111506.59; p < .001$). Significant interactions were   N $\times$ permutations  ($F = 26.52; p < .001$) and N $\times$ features ($F = 273.66; p < .001$), illustrating that the effects of permutations and features were larger for N=1000 than for N=100. Other interactions were not significant (permutations $\times$ features, $p = .58$; N $\times$ permutations $\times$ features, $p  = .08$).\\

\textit{Multi-class LDA}. A three-way analysis of variance (ANOVA) was run on the cross-validation analysis (\prettyref{fig:simulation_results}, bottom left). A continuous variable (features) and two categorical variables, samples N (100 or 1000) and classes (5, 10), were used as predictors, and relative efficiency was used as dependent variable.
There were significant main effects for N ($F = 1023.97; p < .001$), features ($F = 38270.22; p < .001$) but not for classes ($p = .15$). There was a significant 
 features $\times$ N interaction ($F = 125.74; p < .001$) signifying a smaller effect of features for larger N. The other interactions were not significant (features $\times$ classes, $p = .1$; N $\times$ classes, $p = .86$, features $\times$ N $\times$ classes, $p =.462$). 

A separate three-way ANOVA was performed on the permutations data (\prettyref{fig:simulation_results}, bottom right) using N, permutations, and features as predictors. There were significant main effects of N ($F = 366.2; p < .001$), permutations ($F = 27.4; p < .001$), and features ($F = 16970.31; p < .001$).
Again, there was a significant N $\times$ features interaction ($F =  24.93; p < .001$) signifying a smaller effect of features for larger N. The other interactions were not significant (N $\times$ permutations, $p = .13$; permutations $\times$ features, $p = .35$, N $\times$ permutations $\times$ features, $p =.06$).

\subsection{EEG/MEG data}

\begin{figure}
\centering\includegraphics[width=1\linewidth]{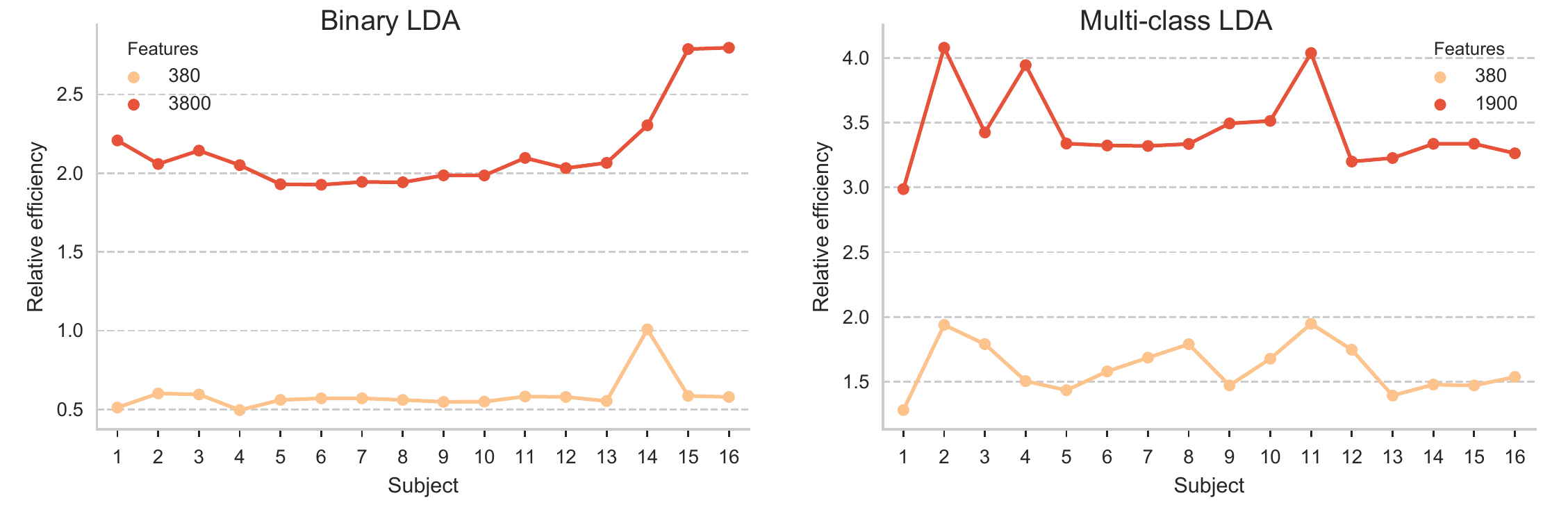}
\caption{Results of the permutation analysis of the EEG/MEG dataset using 100 permutations. Relative efficiency is plotted for each subject (x-axis), and two different numbers of features, with different panels for binary LDA and multi-class LDA.}
\label{fig:WakemanHenson}
\end{figure}

Results are depicted in \prettyref{fig:WakemanHenson}. The data for binary LDA and multi-class LDA were combined into a single two-way ANOVA model. Features (small = 380, large = 3800/1900) and type of classifier (binary LDA, multi-class LDA) were used as predictors. There were significant main effects of features ($F = 826.04; p < .001$) and classifier ($F = 388.2; p < .001$). Moreover, there was a significant features $\times$ classifier interaction ($F = 6.01; p = .017$) signifying a larger effect of features for multi-class LDA than for binary LDA.


\section{Discussion}

Due to its robustness and competitive speed, regularised LDA is an excellent candidate for classification problems involving many training-testing iterations. The analytical approach explored in this paper boosts the performance of least-squares models and multi-class LDA, particularly for high-dimensional data.

The simulations revealed a persistent speed gain using the analytical approach to cross-validation as opposed to the standard approach wherein the classifier is retrained on every training fold. Relative efficiency increases notably with the number of features. Furthermore, it increases with the number of cross-validation folds, and it decreases when the number of samples increases. Cross-validation was significantly faster, by up to 3 orders of magnitude (1000x faster) for binary LDA, and close to 4 orders of magnitude (10,000x faster) for multi-class LDA.

In line with this, the analytical approach consistently outperformed the standard approach in the EEG/MEG analysis. This was particularly prevalent in multi-class LDA, where for 1900 features, the analytical approach was between 1000x and 10,000x faster than the standard approach.

Why is relative efficiency higher for multi-class LDA than for binary LDA? A possible explanation is that multi-class LDA is more involved computationally since a generalised eigenvalue problem needs to be solved, whereas binary LDA requires only matrix inversion. Crucially, the analytical approach requires only a single matrix inversion in both  the binary and the multi-class case. This warrants a larger computational benefit for multi-class LDA.

\subsection{Is it just a trade-off between samples and features?}
It is worth noting that the analytical approach does not simply trade in the number of samples (N) for the number of features (P). If both quantities are equal, e.g. N = P = 1000, binary LDA is about 10x faster than the standard approach for 10-fold cross-validation and about 100x faster for leave-one-out. For multi-class LDA and 10-fold cross-validation, relative efficiency is close to 3 (almost 1000x faster).

The complexity calculations of the analytical approach grows cubically with the number of test samples whereas the standard method grows cubically with the number of features.
Consequently, the standard method and the analytical approach are at parity when
the number of test samples is roughly equal to the number of features, i.e. $ N / K \approx P$. One can deduce the rule of thumb that it is beneficial to use the analytical approach in cross-validation as soon as $P > N / K$. The approach becomes more efficient as K increases with the upper limit being leave-one-out (K = N).

\subsection{What is the practical use of the analytical approach?}

The analytical approach yields a significant increase in speed, but is this practically relevant for typical neuroimaging analyses? There are a number of scenarios in modern neuroimaging analyses wherein a large number of training-testing iterations is performed and hence the approach developed here can be useful:

\begin{itemize}
\item \textit{Multi-dimensional data}. Sometimes, statistical analysis is repeatedly performed along multiple data dimensions. For instance, in time-frequency data, a classifier may be validated for every combination of time point and frequency. In time generalisation, a classifier is trained and tested at every combination of time points in a trial. In searchlight analysis \cite{Kriegeskorte2006Information-basedMapping}, a classifier is validated on a local neighbourhood centered on a voxel, and this operation is repeated for all voxels.
\item \textit{Condition-rich designs}. Some experimental designs, often used in the context of Representational Similarity Analysis (RSA) \cite{Kriegeskorte2008}, feature a large number of stimulus conditions. To build the Representational Dissimilarity Matrix, distances between each pair of conditions are required. Hence with C conditions, $C(C-1)/2$ cross-validations are required for every subject. Initially RSA was based on simple Pearson correlation between samples, but more recent work has increasingly focused on classifier-based approaches, including LDA classification accuracy and LDA-related measures such as Linear Discriminant Contrast (LDC) \cite{Walther2016ReliabilityAnalysis,Diedrichsen2017RepresentationalAnalysis}. 
\item \textit{Permutation testing}. For permutation testing, a classification regime needs to be repeated thousands of times. For instance, \citet{Stelzer2013} developed a cluster test for fMRI data  that involved repeated 100 classification analyses with permuted class labels for each searchlight position and every subject. The results were passed on to the second level to perform group inference. Similarly, in \citet{Allefeld2016ValidInference}, permutations at the subject level are computed for deriving a minimum-statistic used in group inference.
\end{itemize}

\subsection{LDA vs. other least-squares approaches}

Although the analyses presented in this paper focus on LDA, all results readily extend to other least-squares methods such as linear regression and ridge regression. If the vector of class labels is replaced by a vector of continuous responses, then all equations and results apply equally. Furthermore, since multi-class LDA is closely related to Canonical Correlation Analysis (CCA) \cite{Hastie1995PenalizedAnalysis}, speed-ups for CCA might be possible using a similar approach. 

\subsection{Analytical approach vs. kernel methods}

Kernel-based methods such as Support Vector Machines \cite{Cortes1995Support-VectorNetworks} are based on a samples $\times$ samples kernel matrix $\mathbf{K}$ that can be thought of as representing pair-wise similarities between samples according to some non-linear similarity measure. The calculation of the kernel matrix is affected by the number of features, but once the matrix it is available, optimisation algorithms such as Dual Coordinate Descent for SVM \cite{Hsieh2008ASVM} directly operate on the samples dimension rather than the feature dimension. For linear SVM, the kernel matrix comprises standard dot products, $\mathbf{K}_{ij} = \x_i^\top\x_j$. There is a close relationship between the linear kernel and the hat matrix which consists of the entries $\H_{ij} = \x_i^\top(\X^\top\X + \lambda\I)^{-1}\,\x_j$. For $\lambda>0$, this quantity is positive-definite and hence a valid dot product. In other words, the hat matrix is simply a linear kernel whereby the samples have been pre-whitened with respect to the regularised scatter matrix. If the covariance of the samples is normalised and spherical, we have $\H = \mathbf{K}$.

However, kernel methods such as SVM require iterative optimisation algorithms as well as optimisation of hyperparameters. The analytical formula for least-squares methods and multi-class LDA hence yields a computational advantage in many cases.

\subsection{What about big data?}

Due to increasing levels of data sharing and large-scale studies, cognitive neuroscience is on the verge of becoming a big data science \cite{Poldrack2014MakingNeuroimaging,Turk-Browne2013FunctionalBrain,Ferguson2014BigNeuroscience,Choudhury2014BigGenomics.}. In a big data setting, both the number of samples and the number of features is extremely large. This poses a challenge to all statistical learning approaches. For least-squares models, this challenge is admittedly not resolved with the present contribution. However, the following measures can be used to cope with either too many samples or too many features.

\begin{itemize}
\item \textit{Too many samples}. The principal problem is that for a very large number of samples (e.g. $>100,000$) it might be impossible to store the hat matrix in memory. Since a kernel matrix has the same size as the hat matrix, a similar problem occurs in the optimisation of SVMs. In SVM, on-the-fly calculation of the required kernel matrix entries has been proposed as a solution \cite{Hsieh2008ASVM}. If the number of features is small enough, the  matrix $\X^\top\X$ can be stored in memory and submatrices of the hat matrix can be calculated on the fly. Furthermore, the submatrices $\I - \H_\Te$ that need to be inverted   are roughly of size N/K. Hence, one can always find a K large enough such that these matrices are small enough to be invertible efficiently.
\item \textit{Too many features}. If the number of features is too large, it is impossible to store the scatter matrix $\X^\top\X$ in memory. Random projections can offer a solution to this problem. There is evidence that if $\X\in\R^{N\times P}$ is multiplied by a sparse matrix $\mathbf{A}\in\R^{P\times Q}$ with $Q \ll P$, the covariance structure of the original data is approximately preserved in the smaller, sparsified matrix $\X\mathbf{A}\in\R^{N\times Q}$ \cite{Bingham2001RandomReduction}. This matrix can then be used instead of the scatter matrix.
\end{itemize}

An alternative approach that deals with both issues simultaneously is ensemble learning \cite{Hastie2009}, wherein a large number of statistical models called weak learners is trained in parallel. Each model uses a subset of features and a subset of samples. If these subsets are small enough, even large datasets can be digested by the ensemble. Furthermore, since each weak learner is trained independently of the others, ensemble learning can be efficiencly parallelised on compute clusters.

\subsection{Conclusion}
For least-squares methods and multi-class LDA, an analytical approach to cross-validation allows for an increase of computation speed up to several orders of magnitude. The analytical approach  enables least-squares methods and multi-class LDA to be used in high-dimensional feature spaces, particularly in the $P \gg N$ setting (many features, few samples) often encountered in neuroimaging data. Target applications in modern neuroimaging studies include multi-dimensional datasets, Representational Similarity Analysis, and permutation testing.



\appendix

\section{Relationship between linear regression and binary LDA}\label{app:app-regression-binary-fda}

The regression problem \prettyref{eq:LDA-regression-problem} leads to the normal equations

\begin{align}
\begin{split}
\label{eq:normal-equations-Xa}
\Xa^\top\Xa\,\Beta = \Xa\y
\end{split}
\end{align}

Recall that $\Xa$ is the augmented data matrix consisting of the original data $\X$ and a column of 1's. Without loss of generality, one can assume that the samples in the data matrix $\X$ have been arranged as $\X = [\X_1; \X_2]$ such that samples corresponding to class 1 come first and samples corresponding to class 2 come last. The response vector $\y\in\R^N$ contains the numerical codes for the class labels. Class 1 is represented by $z_1\in\R$, class 1 is represented by $z_2\in\R$, class 2 is represented by $z_1\neq z_2$. Accordingly, $\y$ consists of $N_1$ times $z_1$ followed by $N_2$ times $z_2$. Plugging this into \prettyref{eq:normal-equations-Xa} yields

\begin{align}
\begin{split}
\begin{bmatrix}
\X_1^\top & \X_2^\top\\
\one_{N_1}^\top & \one_{N_2}^\top
\end{bmatrix}
\begin{bmatrix}
\X_1 & \one_{N_1}\\
\X_2 & \one_{N_2}
\end{bmatrix}
\begin{bmatrix}
\w \\ b
\end{bmatrix} =
\begin{bmatrix}
\X_1^\top & \X_2^\top\\
\one_{N_1}^\top & \one_{N_2}^\top
\end{bmatrix}
\begin{bmatrix}
z_1\one_{N_1} \\
z_2\one_{N_2}
\end{bmatrix}
\end{split}
\end{align}

Multiplying the matrices and using $\X^\top\X = \S_w + N_1\m_1\m_1^\top + N_2\m_2\m_2^\top$ one obtains

\begin{align}
\begin{split}
\label{eq:normal-equations2}
\begin{bmatrix}
\S_w + N_1\m_1\m_1^\top + N_2\m_2\m_2^\top  & N\,\mbar\\
N\,\mbar^\top & N
\end{bmatrix}
\begin{bmatrix}
\w \\ b
\end{bmatrix} =
\begin{bmatrix}
N_1 z_1\m_1 + N_2 z_2\m_2\\
N_1 z_1 + N_2 z_2
\end{bmatrix}
\end{split}
\end{align}

with $\m_1, \m_2$, and $\mbar$ as defined in \prettyref{eq:means}. Solving the last row of the equation for b yields $b = N_1 z_1/ N + N_2 z_2/ N - \mbar^\top\w$. Plugging this into the first equation in \prettyref{eq:normal-equations2} yields

\begin{align*}
\begin{split}
(\S_w + N_1\m_1^2 + N_2\m_2^2 - N\,\mbar^2)\,\w = N_1\m_1 -N_2\m_2 - (N_1 z_1 + N_2 z_2)\, \mbar
\end{split}
\end{align*}

where $\m^2$ is short for $\m\m^\top$. Using $N\mbar = N_1\m_1 + N_2\m_2$ and the relation $N_1 - N_1^2/N = (N_1 N_2)/N$ one obtains

\begin{align}
\begin{split}
\label{eq:normal-equations4}
(\S_w + \underbrace{\frac{N_1 N_2}{N}\,(\m_1-\m_2)(\m_1-\m_2)^\top}_{\S_b})\,\w = \frac{2N_1 N_2 \,(z_1-z_2)}{N}(\m_1 - \m_2)
\end{split}
\end{align}

The vector $\S_b\,\w$ is a multiple of $(\m_1-\m_2)$, hence there exists $\alpha\in\R$ such that

\begin{align}
\begin{split}
\label{eq:normal-equations-alpha}
\S_b\,\w = (\frac{2N_1 N_2\,(z_1-z_2)}{N} - \alpha)\ (\m_1 - \m_2)
\end{split}
\end{align}

Inserting \prettyref{eq:normal-equations-alpha}
 in \prettyref{eq:normal-equations4} and assuming that $\S_w$ is regular yields

\begin{align*}
\begin{split}
\S_w\w = \alpha\ (\m_1 - \m_2) \Leftrightarrow \w = \alpha\ \S_w^{-1}(\m_1 - \m_2)
\end{split}
\end{align*}

This proves that $\w$ in the linear regression approach is (up to scaling) identical to the LDA solution. Furthermore, the exact numerical codes $z_1$ and $z_2$ for the classes determine $b$ and the scaling of $\w$, but they do not affect the direction of $\w$.

\section{Ridge regularisation for binary LDA}\label{app:ridge}

In this section, the correspondence between the regularised LDA in \prettyref{eq:w-ridge} and ridge regression solution in \prettyref{eq:LDA-ridge-regression-solution} is established. To simplify the math, it is assumed that in $\y$, class 1 is coded as $+1$ and class 2 is coded as $-1$. The assertion is that regularised LDA can be cast in a least-squares framework using the normal equations

\begin{align}
\begin{split}
\label{eq:ridge-normal-equations-Xa}
(\Xa^\top\Xa + \lambda\I_0)\,\Beta = \Xa\y
\end{split}
\end{align}

where $\I_0$ is defined like in \prettyref{eq:LDA-ridge-regression-solution}. Following the derivation in the previous section, one arrives at

\begin{align}
\begin{split}
\label{eq:ridge-normal-equations1}
\left(
\begin{bmatrix}
\X_1^\top & \X_2^\top\\
\one_{N_1}^\top & \one_{N_2}^\top
\end{bmatrix}
\begin{bmatrix}
\X_1 & \one_{N_1}\\
\X_2 & \one_{N_2}
\end{bmatrix} + \lambda\I_0
\right)
\begin{bmatrix}
\w \\ b
\end{bmatrix} =
\begin{bmatrix}
\X_1^\top & \X_2^\top\\
\one_{N_1}^\top & \one_{N_2}^\top
\end{bmatrix}
\begin{bmatrix}
\one_{N_1} \\
-\one_{N_2}
\end{bmatrix}
\end{split}
\end{align}

and finally

\begin{align}
\begin{split}
\label{eq:ridge-normal-equations2}
\begin{bmatrix}
(\S_w + \lambda\I) + N_1\m_1\m_1^\top + N_2\m_2\m_2^\top  & N\,\mbar\\
N\,\mbar^\top & N
\end{bmatrix}
\begin{bmatrix}
\w \\ b
\end{bmatrix} =
\begin{bmatrix}
N_1\m_1 - N_2\m_2\\
N_1 - N_2
\end{bmatrix}
\end{split}
\end{align}
The rest of the proof follows the approach in the previous section, with $\S_w$ being replaced by $\S_w + \lambda\I$. This proves the normal equations in \prettyref{eq:ridge-normal-equations-Xa} correspond to regularised LDA.

\section{Proof of lemma}\label{app:proofs}

\begin{lemma}\label{lem:evproblem}
Let $\S_b\,\w = \lambda\,\S_w\,\w$ be the generalised eigenvalue problem associated with a binary classification problem with unequal class means ($\m_1\neq\m_2$) and let $\S_w$ be positive definite. Then there is one non-zero eigenvalue $\lambda = N_1\,N_2/N\ (\m_1-\m_2)^\top\S_w^{-1} (\m_1-\m_2) > 0$. The associated eigenvector is proportional to $\S_w^{-1}(\m_1-\m_2)$.
\end{lemma}

\begin{proof}
Define $\Delta := \m_1 - \m_2$ and $\w := \S_w^{-1}\Delta$. Since $\S_w$ is regular, the generalised eigenvalue problem can be written as an ordinary eigenvalue problem  $\S_w^{-1}\S_b\,\w = \lambda\,\w$. Then using \prettyref{eq:Sb-simple} for $\S_b$ one obtains

\[
\S_w^{-1}\S_b\w = \S_w^{-1}(\frac{N_1\,N_2}{N}\Delta\Delta^\top)\w = \S_w^{-1}\Delta(\underbrace{\frac{N_1\,N_2}{N}\,\Delta^\top\S_w^{-1}\Delta}_{:=\lambda}) = \lambda\w ,
\]

hence $\w$ is an eigenvector of $\S_w^{-1}\S_b$ with eigenvalue $\lambda$. Since $\S_w^{-1}$ is positive definite, $\lambda > 0$. Since $\S_w^{-1}\S_b$ is of rank 1, all other eigenvalues are zero.
\end{proof}








\section*{Acknowledgements}

I would like to thank Richard Henson for helpful comments.



\section*{References}

\bibliographystyle{model1-num-names}
\bibliography{Mendeley.bib}

\begin{thebibliography}{46}
\expandafter\ifx\csname natexlab\endcsname\relax\def\natexlab#1{#1}\fi
\providecommand{\bibinfo}[2]{#2}
\ifx\xfnm\relax \def\xfnm[#1]{\unskip,\space#1}\fi
\bibitem[{Mur et~al.(2009)Mur, Bandettini, and Kriegeskorte}]{Mur2009}
\bibinfo{author}{M.~Mur}, \bibinfo{author}{P.~A. Bandettini},
  \bibinfo{author}{N.~Kriegeskorte},
\newblock \bibinfo{title}{{Revealing representational content with
  pattern-information fMRI - An introductory guide}},
\newblock \bibinfo{journal}{Social Cognitive and Affective Neuroscience}
  \bibinfo{volume}{4} (\bibinfo{year}{2009}) \bibinfo{pages}{101--109}.
\bibitem[{Hastie et~al.(2009)Hastie, Tibshirani, and Friedman}]{Hastie2009}
\bibinfo{author}{T.~Hastie}, \bibinfo{author}{R.~Tibshirani},
  \bibinfo{author}{J.~Friedman},
\newblock \bibinfo{title}{{The Elements of Statistical Learning}},
\newblock in: \bibinfo{booktitle}{The Elements of Statistical Learning},
  \bibinfo{publisher}{Springer New York Inc.}, \bibinfo{address}{New York, NY,
  USA}, \bibinfo{year}{2009}.
\bibitem[{Fisher(1936)}]{Fisher1936}
\bibinfo{author}{R.~A. Fisher},
\newblock \bibinfo{title}{{The use of multiple measurements in taxonomic
  problems}},
\newblock \bibinfo{journal}{Annals of Eugenics} \bibinfo{volume}{7}
  (\bibinfo{year}{1936}) \bibinfo{pages}{179--188}.
\bibitem[{Clarke et~al.(2008)Clarke, Ressom, Wang, Xuan, Liu, Gehan, and
  Wang}]{Clarke2008TheData.}
\bibinfo{author}{R.~Clarke}, \bibinfo{author}{H.~W. Ressom},
  \bibinfo{author}{A.~Wang}, \bibinfo{author}{J.~Xuan}, \bibinfo{author}{M.~C.
  Liu}, \bibinfo{author}{E.~A. Gehan}, \bibinfo{author}{Y.~Wang},
\newblock \bibinfo{title}{{The properties of high-dimensional data spaces:
  implications for exploring gene and protein expression data.}},
\newblock \bibinfo{journal}{Nature reviews. Cancer} \bibinfo{volume}{8}
  (\bibinfo{year}{2008}) \bibinfo{pages}{37--49}.
\bibitem[{Wang et~al.(2008)Wang, Miller, and
  Clarke}]{Wang2008ApproachesMicroarrays}
\bibinfo{author}{Y.~Wang}, \bibinfo{author}{D.~J. Miller},
  \bibinfo{author}{R.~Clarke},
\newblock \bibinfo{title}{{Approaches to working in high-dimensional data
  spaces: gene expression microarrays}},
\newblock \bibinfo{journal}{British Journal of Cancer} \bibinfo{volume}{98}
  (\bibinfo{year}{2008}) \bibinfo{pages}{1023--1028}.
\bibitem[{Kriegeskorte et~al.(2008)Kriegeskorte, Mur, and
  Bandettini}]{Kriegeskorte2008}
\bibinfo{author}{N.~Kriegeskorte}, \bibinfo{author}{M.~Mur},
  \bibinfo{author}{P.~Bandettini},
\newblock \bibinfo{title}{{Representational similarity analysis - connecting
  the branches of systems neuroscience}},
\newblock \bibinfo{journal}{Frontiers in systems neuroscience}
  \bibinfo{volume}{2} (\bibinfo{year}{2008}) \bibinfo{pages}{4}.
\bibitem[{Cortes and Vapnik(1995)}]{Cortes1995Support-VectorNetworks}
\bibinfo{author}{C.~Cortes}, \bibinfo{author}{V.~Vapnik},
\newblock \bibinfo{title}{{Support-Vector Networks}},
\newblock \bibinfo{journal}{Machine Learning} \bibinfo{volume}{20}
  (\bibinfo{year}{1995}) \bibinfo{pages}{273--297}.
\bibitem[{Lemm et~al.(2011)Lemm, Blankertz, Dickhaus, and
  M{\"{u}}ller}]{Lemm2011}
\bibinfo{author}{S.~Lemm}, \bibinfo{author}{B.~Blankertz},
  \bibinfo{author}{T.~Dickhaus}, \bibinfo{author}{K.~R. M{\"{u}}ller},
\newblock \bibinfo{title}{{Introduction to machine learning for brain
  imaging}},
\newblock \bibinfo{journal}{NeuroImage} \bibinfo{volume}{56}
  (\bibinfo{year}{2011}) \bibinfo{pages}{387--399}.
\bibitem[{Cawley and Talbot(2003)}]{Cawley2003EfficientClassifiers}
\bibinfo{author}{G.~C. Cawley}, \bibinfo{author}{N.~L. Talbot},
\newblock \bibinfo{title}{{Efficient leave-one-out cross-validation of kernel
  fisher discriminant classifiers}},
\newblock \bibinfo{journal}{Pattern Recognition} \bibinfo{volume}{36}
  (\bibinfo{year}{2003}) \bibinfo{pages}{2585--2592}.
\bibitem[{Cook and Weisberg(1982)}]{Cook1982ResidualsRegression}
\bibinfo{author}{R.~D. Cook}, \bibinfo{author}{S.~Weisberg},
  \bibinfo{title}{{Residuals and influence in regression}},
  \bibinfo{publisher}{Chapman and Hall}, \bibinfo{year}{1982}.
\bibitem[{James et~al.(2013)James, Witten, Hastie, and Tibishirani}]{James2013}
\bibinfo{author}{G.~James}, \bibinfo{author}{D.~Witten},
  \bibinfo{author}{T.~Hastie}, \bibinfo{author}{R.~Tibishirani},
  \bibinfo{title}{{An Introduction to Statistical Learning}},
  \bibinfo{year}{2013}.
\bibitem[{Pahikkala et~al.(2006)Pahikkala, Boberg, and
  Salakoski}]{Pahikkala2006FastLeast-Squares}
\bibinfo{author}{T.~Pahikkala}, \bibinfo{author}{J.~Boberg},
  \bibinfo{author}{T.~Salakoski},
\newblock \bibinfo{title}{{Fast n-Fold Cross-Validation for Regularized
  Least-Squares}},
\newblock in: \bibinfo{booktitle}{Proceedings of SCAI’0}, pp.
  \bibinfo{pages}{83--90}.
\bibitem[{Rao et~al.(2008)Rao, Fung, and Rosales}]{Rao2008OnEvaluation}
\bibinfo{author}{R.~B. Rao}, \bibinfo{author}{G.~Fung},
  \bibinfo{author}{R.~Rosales},
\newblock \bibinfo{title}{{On the Dangers of Cross-Validation. An Experimental
  Evaluation}},
\newblock in: \bibinfo{booktitle}{roceedings of the 2008 SIAM International
  Conference on Data Mining}, pp. \bibinfo{pages}{588--596}.
\bibitem[{Friedman(1989)}]{Friedman1989}
\bibinfo{author}{J.~H. Friedman},
\newblock \bibinfo{title}{{Regularized Discriminant Analysis}},
\newblock \bibinfo{journal}{Journal of the American Statistical Association}
  \bibinfo{volume}{84} (\bibinfo{year}{1989}) \bibinfo{pages}{165--175}.
\bibitem[{Blankertz et~al.(2011)Blankertz, Lemm, Treder, Haufe, and
  M{\"{u}}ller}]{Blankertz2011}
\bibinfo{author}{B.~Blankertz}, \bibinfo{author}{S.~Lemm},
  \bibinfo{author}{M.~Treder}, \bibinfo{author}{S.~Haufe},
  \bibinfo{author}{K.~R. M{\"{u}}ller},
\newblock \bibinfo{title}{{Single-trial analysis and classification of ERP
  components - A tutorial}},
\newblock \bibinfo{journal}{NeuroImage} \bibinfo{volume}{56}
  (\bibinfo{year}{2011}) \bibinfo{pages}{814--825}.
\bibitem[{Li et~al.(2006)Li, Zhu, and Ogihara}]{Li2006UsingInvestigation}
\bibinfo{author}{T.~Li}, \bibinfo{author}{S.~Zhu},
  \bibinfo{author}{M.~Ogihara},
\newblock \bibinfo{title}{{Using discriminant analysis for multi-class
  classification: an experimental investigation}},
\newblock \bibinfo{journal}{Knowledge and Information Systems}
  \bibinfo{volume}{10} (\bibinfo{year}{2006}) \bibinfo{pages}{453--472}.
\bibitem[{Allefeld et~al.(2016)Allefeld, G{\"{o}}rgen, and
  Haynes}]{Allefeld2016ValidInference}
\bibinfo{author}{C.~Allefeld}, \bibinfo{author}{K.~G{\"{o}}rgen},
  \bibinfo{author}{J.-D. Haynes},
\newblock \bibinfo{title}{{Valid population inference for information-based
  imaging: From the second-level t -test to prevalence inference}},
\newblock \bibinfo{journal}{NeuroImage} \bibinfo{volume}{141}
  (\bibinfo{year}{2016}) \bibinfo{pages}{378--392}.
\bibitem[{Stelzer et~al.(2013)Stelzer, Chen, and Turner}]{Stelzer2013}
\bibinfo{author}{J.~Stelzer}, \bibinfo{author}{Y.~Chen},
  \bibinfo{author}{R.~Turner},
\newblock \bibinfo{title}{{Statistical inference and multiple testing
  correction in classification-based multi-voxel pattern analysis (MVPA):
  Random permutations and cluster size control}},
\newblock \bibinfo{journal}{NeuroImage} \bibinfo{volume}{65}
  (\bibinfo{year}{2013}) \bibinfo{pages}{69--82}.
\bibitem[{Ojala and Garriga(2010)}]{Ojala2010}
\bibinfo{author}{M.~Ojala}, \bibinfo{author}{G.~C. Garriga},
\newblock \bibinfo{title}{{Permutation Tests for Studying Classifier
  Performance}},
\newblock \bibinfo{journal}{Journal of Machine Learning Research}
  \bibinfo{volume}{11} (\bibinfo{year}{2010}) \bibinfo{pages}{1833--1863}.
\bibitem[{Salzberg(1997)}]{Salzberg1997}
\bibinfo{author}{S.~L. Salzberg},
\newblock \bibinfo{title}{{On Comparing Classifiers: Pitfalls to Avoid and a
  Recommended Approach}},
\newblock \bibinfo{journal}{Data Mining and Knowledge Discovery}
  \bibinfo{volume}{1} (\bibinfo{year}{1997}) \bibinfo{pages}{317--327}.
\bibitem[{Jamalabadi et~al.(2016)Jamalabadi, Alizadeh, Sch{\"{o}}nauer,
  Leibold, and Gais}]{Jamalabadi2016}
\bibinfo{author}{H.~Jamalabadi}, \bibinfo{author}{S.~Alizadeh},
  \bibinfo{author}{M.~Sch{\"{o}}nauer}, \bibinfo{author}{C.~Leibold},
  \bibinfo{author}{S.~Gais},
\newblock \bibinfo{title}{{Classification based hypothesis testing in
  neuroscience: Below-chance level classification rates and overlooked
  statistical properties of linear parametric classifiers}},
\newblock \bibinfo{journal}{Human Brain Mapping} \bibinfo{volume}{37}
  (\bibinfo{year}{2016}) \bibinfo{pages}{1842--1855}.
\bibitem[{Hastie et~al.(1995)Hastie, Buja, and
  Tibshirani}]{Hastie1995PenalizedAnalysis}
\bibinfo{author}{T.~Hastie}, \bibinfo{author}{A.~Buja},
  \bibinfo{author}{R.~Tibshirani},
\newblock \bibinfo{title}{{Penalized Discriminant Analysis}},
\newblock \bibinfo{journal}{The Annals of Statistics} \bibinfo{volume}{23}
  (\bibinfo{year}{1995}) \bibinfo{pages}{73--102}.
\bibitem[{Rao(1948)}]{Rao1948TheClassification}
\bibinfo{author}{C.~R. Rao}, \bibinfo{title}{{The Utilization of Multiple
  Measurements in Problems of Biological Classification}},
  \bibinfo{year}{1948}.
\bibitem[{Wakeman and Henson(2015)}]{Wakeman2015ADataset}
\bibinfo{author}{D.~G. Wakeman}, \bibinfo{author}{R.~N. Henson},
\newblock \bibinfo{title}{{A multi-subject, multi-modal human neuroimaging
  dataset}},
\newblock \bibinfo{journal}{Scientific Data} \bibinfo{volume}{2}
  (\bibinfo{year}{2015}) \bibinfo{pages}{150001}.
\bibitem[{Bishop(2007)}]{Bishop2007}
\bibinfo{author}{C.~M. Bishop},
\newblock \bibinfo{title}{{Pattern Recognition and Machine Learning}},
\newblock \bibinfo{journal}{Journal of Electronic Imaging} \bibinfo{volume}{16}
  (\bibinfo{year}{2007}) \bibinfo{pages}{049901}.
\bibitem[{Duda et~al.(1998)Duda, Hart, and Stork}]{Duda1998}
\bibinfo{author}{R.~O. Duda}, \bibinfo{author}{P.~E. Hart},
  \bibinfo{author}{D.~G. Stork}, \bibinfo{title}{{Pattern classification}},
  \bibinfo{year}{1998}.
\bibitem[{Treder et~al.(2016)Treder, Porbadnigk, Shahbazi~Avarvand,
  M{\"{u}}ller, and Blankertz}]{Treder2016}
\bibinfo{author}{M.~S. Treder}, \bibinfo{author}{A.~K. Porbadnigk},
  \bibinfo{author}{F.~Shahbazi~Avarvand}, \bibinfo{author}{K.-R. M{\"{u}}ller},
  \bibinfo{author}{B.~Blankertz},
\newblock \bibinfo{title}{{The LDA beamformer: Optimal estimation of ERP source
  time series using linear discriminant analysis}},
\newblock \bibinfo{journal}{NeuroImage} \bibinfo{volume}{129}
  (\bibinfo{year}{2016}) \bibinfo{pages}{279--291}.
\bibitem[{van Vliet et~al.(2016)van Vliet, Chumerin, De~Deyne, Wiersema, Fias,
  Storms, and Van~Hulle}]{vanVliet2016Single-TrialBeamformer}
\bibinfo{author}{M.~van Vliet}, \bibinfo{author}{N.~Chumerin},
  \bibinfo{author}{S.~De~Deyne}, \bibinfo{author}{J.~R. Wiersema},
  \bibinfo{author}{W.~Fias}, \bibinfo{author}{G.~Storms},
  \bibinfo{author}{M.~M. Van~Hulle},
\newblock \bibinfo{title}{{Single-Trial ERP Component Analysis Using a
  Spatiotemporal LCMV Beamformer}},
\newblock \bibinfo{journal}{IEEE Transactions on Biomedical Engineering}
  \bibinfo{volume}{63} (\bibinfo{year}{2016}) \bibinfo{pages}{55--66}.
\bibitem[{van Vliet et~al.(2017)van Vliet, Van~Hulle, and
  Salmelin}]{vanVliet2017ExploringPotentials}
\bibinfo{author}{M.~van Vliet}, \bibinfo{author}{M.~M. Van~Hulle},
  \bibinfo{author}{R.~Salmelin},
\newblock \bibinfo{title}{{Exploring the Organization of Semantic Memory
  through Unsupervised Analysis of Event-related Potentials}},
\newblock \bibinfo{journal}{Journal of Cognitive Neuroscience}
  (\bibinfo{year}{2017}) \bibinfo{pages}{1--12}.
\bibitem[{Mika(2002)}]{Mika2002KernelDiscriminants}
\bibinfo{author}{S.~Mika}, \bibinfo{title}{{Kernel Fisher Discriminants}},
  Ph.D. thesis, Technische Universit{\"{a}}t Berlin, \bibinfo{year}{2002}.
\bibitem[{Zhang et~al.(2010)Zhang, Dai, Xu, and
  Jordan}]{Zhang2010RegularizedBeyond}
\bibinfo{author}{Z.~Zhang}, \bibinfo{author}{G.~Dai}, \bibinfo{author}{C.~Xu},
  \bibinfo{author}{M.~I. Jordan},
\newblock \bibinfo{title}{{Regularized Discriminant Analysis, Ridge Regression
  and Beyond}},
\newblock \bibinfo{journal}{Journal of Machine Learning Research}
  \bibinfo{volume}{11} (\bibinfo{year}{2010}) \bibinfo{pages}{2199--2228}.
\bibitem[{Hoaglin and Welsch(1978)}]{Hoaglin1978TheANOVA}
\bibinfo{author}{D.~C. Hoaglin}, \bibinfo{author}{R.~E. Welsch},
\newblock \bibinfo{title}{{The Hat Matrix in Regression and ANOVA}},
\newblock \bibinfo{journal}{The American Statistician} \bibinfo{volume}{32}
  (\bibinfo{year}{1978}) \bibinfo{pages}{17--22}.
\bibitem[{Tikhonov and Arsenin(1977)}]{Tikhonov1977SolutionsProblems}
\bibinfo{author}{A.~N. A.~N. Tikhonov}, \bibinfo{author}{V.~I. V.~I. Arsenin},
  \bibinfo{title}{{Solutions of ill-posed problems}},
  \bibinfo{publisher}{Winston}, \bibinfo{year}{1977}.
\bibitem[{Ng(2004)}]{Ng2004FeatureInvariance}
\bibinfo{author}{A.~Ng},
\newblock \bibinfo{title}{{Feature selection, L1 vs. L2 regularization, and
  rotational invariance}},
\newblock \bibinfo{journal}{Twenty-first international conference on Machine
  learning - ICML '04}  (\bibinfo{year}{2004}) \bibinfo{pages}{78}.
\bibitem[{Ye and {Jieping}(2007)}]{Ye2007LeastAnalysis}
\bibinfo{author}{J.~Ye}, \bibinfo{author}{{Jieping}},
\newblock \bibinfo{title}{{Least squares linear discriminant analysis}},
\newblock in: \bibinfo{booktitle}{Proceedings of the 24th international
  conference on Machine learning - ICML '07}, \bibinfo{publisher}{ACM Press},
  \bibinfo{address}{New York, New York, USA}, \bibinfo{year}{2007}, pp.
  \bibinfo{pages}{1087--1093}.
\bibitem[{Park and Park(2005)}]{Park2005ASolution}
\bibinfo{author}{C.~H. Park}, \bibinfo{author}{H.~Park},
\newblock \bibinfo{title}{{A Relationship between Linear Discriminant Analysis
  and the Generalized Minimum Squared Error Solution}},
\newblock \bibinfo{journal}{SIAM Journal on Matrix Analysis and Applications}
  \bibinfo{volume}{27} (\bibinfo{year}{2005}) \bibinfo{pages}{474--492}.
\bibitem[{Oostenveld et~al.(2011)Oostenveld, Fries, Maris, and
  Schoffelen}]{Oostenveld2011}
\bibinfo{author}{R.~Oostenveld}, \bibinfo{author}{P.~Fries},
  \bibinfo{author}{E.~Maris}, \bibinfo{author}{J.-M. Schoffelen},
\newblock \bibinfo{title}{{FieldTrip: Open Source Software for Advanced
  Analysis of MEG, EEG, and Invasive Electrophysiological Data}},
\newblock \bibinfo{journal}{Computational Intelligence and Neuroscience}
  \bibinfo{volume}{2011} (\bibinfo{year}{2011}) \bibinfo{pages}{1--9}.
\bibitem[{Kriegeskorte et~al.(2006)Kriegeskorte, Goebel, and
  Bandettini}]{Kriegeskorte2006Information-basedMapping}
\bibinfo{author}{N.~Kriegeskorte}, \bibinfo{author}{R.~Goebel},
  \bibinfo{author}{P.~Bandettini},
\newblock \bibinfo{title}{{Information-based functional brain mapping}},
\newblock \bibinfo{journal}{Proceedings of the National Academy of Sciences}
  \bibinfo{volume}{103} (\bibinfo{year}{2006}) \bibinfo{pages}{3863--3868}.
\bibitem[{Walther et~al.(2016)Walther, Nili, Ejaz, Alink, Kriegeskorte, and
  Diedrichsen}]{Walther2016ReliabilityAnalysis}
\bibinfo{author}{A.~Walther}, \bibinfo{author}{H.~Nili},
  \bibinfo{author}{N.~Ejaz}, \bibinfo{author}{A.~Alink},
  \bibinfo{author}{N.~Kriegeskorte}, \bibinfo{author}{J.~Diedrichsen},
\newblock \bibinfo{title}{{Reliability of dissimilarity measures for
  multi-voxel pattern analysis}},
\newblock \bibinfo{journal}{NeuroImage} \bibinfo{volume}{137}
  (\bibinfo{year}{2016}) \bibinfo{pages}{188--200}.
\bibitem[{Diedrichsen and
  Kriegeskorte(2017)}]{Diedrichsen2017RepresentationalAnalysis}
\bibinfo{author}{J.~Diedrichsen}, \bibinfo{author}{N.~Kriegeskorte},
\newblock \bibinfo{title}{{Representational models: A common framework for
  understanding encoding, pattern-component, and representational-similarity
  analysis}},
\newblock \bibinfo{journal}{PLOS Computational Biology} \bibinfo{volume}{13}
  (\bibinfo{year}{2017}) \bibinfo{pages}{e1005508}.
\bibitem[{Hsieh et~al.(2008)Hsieh, Chang, Lin, Keerthi, and
  Sundararajan}]{Hsieh2008ASVM}
\bibinfo{author}{C.-J. Hsieh}, \bibinfo{author}{K.-W. Chang},
  \bibinfo{author}{C.-J. Lin}, \bibinfo{author}{S.~S. Keerthi},
  \bibinfo{author}{S.~Sundararajan},
\newblock \bibinfo{title}{{A dual coordinate descent method for large-scale
  linear SVM}},
\newblock in: \bibinfo{booktitle}{Proceedings of the 25th international
  conference on Machine learning - ICML '08}, \bibinfo{publisher}{ACM Press},
  \bibinfo{address}{New York, New York, USA}, \bibinfo{year}{2008}, pp.
  \bibinfo{pages}{408--415}.
\bibitem[{Poldrack and Gorgolewski(2014)}]{Poldrack2014MakingNeuroimaging}
\bibinfo{author}{R.~A. Poldrack}, \bibinfo{author}{K.~J. Gorgolewski},
\newblock \bibinfo{title}{{Making big data open: data sharing in
  neuroimaging}},
\newblock \bibinfo{journal}{Nature Neuroscience} \bibinfo{volume}{17}
  (\bibinfo{year}{2014}) \bibinfo{pages}{1510--1517}.
\bibitem[{Turk-Browne(2013)}]{Turk-Browne2013FunctionalBrain}
\bibinfo{author}{N.~B. Turk-Browne},
\newblock \bibinfo{title}{{Functional Interactions as Big Data in the Human
  Brain}},
\newblock \bibinfo{journal}{Science} \bibinfo{volume}{342}
  (\bibinfo{year}{2013}) \bibinfo{pages}{580--584}.
\bibitem[{Ferguson et~al.(2014)Ferguson, Nielson, Cragin, Bandrowski, and
  Martone}]{Ferguson2014BigNeuroscience}
\bibinfo{author}{A.~R. Ferguson}, \bibinfo{author}{J.~L. Nielson},
  \bibinfo{author}{M.~H. Cragin}, \bibinfo{author}{A.~E. Bandrowski},
  \bibinfo{author}{M.~E. Martone},
\newblock \bibinfo{title}{{Big data from small data: data-sharing in the 'long
  tail' of neuroscience}},
\newblock \bibinfo{journal}{Nature Neuroscience} \bibinfo{volume}{17}
  (\bibinfo{year}{2014}) \bibinfo{pages}{1442--1447}.
\bibitem[{Choudhury et~al.(2014)Choudhury, Fishman, McGowan, and
  Juengst}]{Choudhury2014BigGenomics.}
\bibinfo{author}{S.~Choudhury}, \bibinfo{author}{J.~R. Fishman},
  \bibinfo{author}{M.~L. McGowan}, \bibinfo{author}{E.~T. Juengst},
\newblock \bibinfo{title}{{Big data, open science and the brain: lessons
  learned from genomics.}},
\newblock \bibinfo{journal}{Frontiers in human neuroscience}
  \bibinfo{volume}{8} (\bibinfo{year}{2014}) \bibinfo{pages}{239}.
\bibitem[{Bingham and Mannila(2001)}]{Bingham2001RandomReduction}
\bibinfo{author}{E.~Bingham}, \bibinfo{author}{H.~Mannila},
\newblock \bibinfo{title}{{Random projection in dimensionality reduction}},
\newblock in: \bibinfo{booktitle}{Proceedings of the seventh ACM SIGKDD
  international conference on Knowledge discovery and data mining - KDD '01},
  \bibinfo{publisher}{ACM Press}, \bibinfo{address}{New York, New York, USA},
  \bibinfo{year}{2001}, pp. \bibinfo{pages}{245--250}.

\end{thebibliography}







\end{document}